\pdfoutput=1
\documentclass[10pt]{article}
\usepackage{qclock}
\usepackage{mathtools}
\usepackage{array}
\usepackage{aliascnt}
\usepackage[capitalise,nameinlink]{cleveref}
\newcommand{\NumEps}{0.39}
\newcommand{\NumR}{1.386}
\newcommand{\NumG}{0.1540}
\newcommand{\NumTfluid}{16.65}
\newcommand{\NumODEdouble}{3.94}
\newcommand{\NumTfluidC}{19.00}
\newcommand{\NumTlevel}{11.54}
\newcommand{\NumFirstDblA}{3.97}
\newcommand{\NumFirstDblAb}{4.17}
\newcommand{\NumFirstDblB}{4.36}
\newcommand{\NumFirstDblC}{4.50}
\newcommand{\NumFirstDblD}{14.95}
\newcommand{\NumEllA}{0.0039}
\newcommand{\NumEzA}{16.7}
\newcommand{\NumCVA}{0.044}
\newcommand{\NumQloA}{15.5}
\newcommand{\NumQhiA}{17.9}
\newcommand{\NumPtwelveA}{1.2\times10^{-13}}
\newcommand{\NumNstarA}{130}
\newcommand{\NumErrPrevA}{0.0519}
\newcommand{\NumErrNA}{0.0499}
\newcommand{\NumWA}{0.604}
\newcommand{\NumETnA}{6.6}
\newcommand{\NumRemA}{10.07}
\newcommand{\NumRemLoA}{9.36}
\newcommand{\NumRemHiA}{10.82}
\newcommand{\NumEllB}{0.037}
\newcommand{\NumEzB}{17.0}
\newcommand{\NumCVB}{0.136}
\newcommand{\NumQloB}{13.4}
\newcommand{\NumQhiB}{21.0}
\newcommand{\NumPtwelveB}{0.007}
\newcommand{\NumNstarB}{30}
\newcommand{\NumErrPrevB}{0.0538}
\newcommand{\NumErrNB}{0.0451}
\newcommand{\NumWB}{0.328}
\newcommand{\NumETnB}{11.4}
\newcommand{\NumRemB}{5.56}
\newcommand{\NumRemLoB}{4.40}
\newcommand{\NumRemHiB}{6.88}
\newcommand{\NumEllC}{0.270}
\newcommand{\NumEzC}{19.0}
\newcommand{\NumCVC}{0.367}
\newcommand{\NumQloC}{9.7}
\newcommand{\NumQhiC}{31.9}
\newcommand{\NumPtwelveC}{0.140}
\newcommand{\NumNstarC}{9}
\newcommand{\NumErrPrevC}{0.0641}
\newcommand{\NumErrNC}{0.0332}
\newcommand{\NumWC}{0.088}
\newcommand{\NumETnC}{17.3}
\newcommand{\NumRemC}{1.67}
\newcommand{\NumRemLoC}{0.85}
\newcommand{\NumRemHiC}{2.80}
\newcommand{\NumEllD}{0.898}
\newcommand{\NumEzD}{25.2}
\newcommand{\NumCVD}{0.649}
\newcommand{\NumQloD}{6.9}
\newcommand{\NumQhiD}{57.0}
\newcommand{\NumPtwelveD}{0.196}
\newcommand{\NumNstarD}{5}
\newcommand{\NumErrPrevD}{0.0644}
\newcommand{\NumErrND}{0.0136}
\newcommand{\NumWD}{0.011}
\newcommand{\NumETnD}{24.9}
\newcommand{\NumRemD}{0.28}
\newcommand{\NumRemLoD}{0.08}
\newcommand{\NumRemHiD}{0.64}
\newcommand{\NumMeanUnits}{4.22}
\newcommand{\NumZlo}{2.148}
\newcommand{\NumZhi}{7.098}
\newcommand{\NumLeadTwelve}{5.59}
\newcommand{\NumETnine}{17.3}
\newcommand{\NumEzetaC}{19.00}
\newcommand{\NumLAN}{3.19}

\newcommand{\NumPexplode}{0.743}
\newcommand{\NumPriorMedian}{17.3}
\newcommand{\NumPgtFive}{0.150}
\newcommand{\NumPgtTen}{0.074}
\newcommand{\NumParetoC}{8.67}
\newcommand{\NumAmpHone}{3.4\times10^{-15}}
\newcommand{\NumAmpTwelve}{3.1\times10^{-3}}
\newcommand{\NumAmpLogTen}{11}
\newcommand{\NumLagDouble}{6.37}
\newcommand{\NumGammaStar}{0.279}
\newcommand{\NumShareTen}{47}
\newcommand{\NumShareThousand}{85}
\newcommand{\NumPathFrac}{0.0988}
\newcommand{\NumPathFracSE}{0.00011}
\newcommand{\NumRatioMeans}{0.0878}

\newcommand{\qcSharedTheorem}[2]{%
  \newaliascnt{#1}{theorem}\newtheorem{#1}[#1]{#2}\aliascntresetthe{#1}}
\theoremstyle{qcstatement}
\newtheorem{theorem}{Theorem}[section]
\qcSharedTheorem{proposition}{Proposition}
\qcSharedTheorem{lemma}{Lemma}
\qcSharedTheorem{corollary}{Corollary}
\qcRailEnvironment{theorem}{qcEmber}
\qcRailEnvironment{proposition}{qcEmber}
\qcRailEnvironment{lemma}{qcEmber}
\qcRailEnvironment{corollary}{qcEmber}
\theoremstyle{qcdefinition}
\qcSharedTheorem{definition}{Definition}
\qcRailEnvironment{definition}{qcSlate}
\theoremstyle{qcremark}
\qcSharedTheorem{remark}{Remark}
\qcSharedTheorem{example}{Example}
\crefname{theorem}{Theorem}{Theorems}
\crefname{proposition}{Proposition}{Propositions}
\crefname{lemma}{Lemma}{Lemmas}
\crefname{corollary}{Corollary}{Corollaries}
\crefname{definition}{Definition}{Definitions}
\crefname{remark}{Remark}{Remarks}
\crefname{example}{Example}{Examples}
\crefname{qcbox}{Box}{Boxes}
\numberwithin{equation}{section}

\newcommand{\E}{\mathbb{E}}
\renewcommand{\P}{\mathbb{P}}
\newcommand{\R}{\mathbb{R}}
\newcommand{\N}{\mathbb{N}}
\newcommand{\C}{\mathbb{C}}
\newcommand{\F}{\mathcal{F}}
\newcommand{\eps}{\varepsilon}
\newcommand{\zb}{\bar\zeta}
\newcommand{\dd}{\,\mathrm{d}}
\newcommand{\KL}{\mathrm{KL}}
\newcommand{\sech}{\operatorname{sech}}
\DeclareMathOperator{\Var}{Var}
\DeclareMathOperator{\CV}{CV}
\DeclareMathOperator{\Exp}{Exp}
\DeclareMathOperator{\GammaDist}{Gamma}
\DeclareMathOperator{\BetaDist}{Beta}
\newcommand{\err}{\mathrm{err}}

\hypersetup{pdftitle={The Singularity Has a Distribution: Exact Timing Laws and Detection Limits for a Stochastic Intelligence Explosion},
  pdfauthor={Xiaoyu Li, Zheng Gao, Chiwun Yang, Zhizhou Sha}}

\qcShortTitle{The Singularity Has a Distribution}
\qcKicker{Working paper \textperiodcentered{} 13 October 2026}
\title{The Singularity Has a Distribution}
\subtitle{Exact Timing Laws and Detection Limits for a Stochastic Intelligence Explosion}
\author{Xiaoyu Li\qcAff{1}\qcSep Zheng Gao\qcAff{1}\qcSep Chiwun Yang\qcAff{2}\qcSep Zhizhou Sha\qcAff{3}}
\qcAffiliations{\qcAff{1}University of New South Wales\quad
  \qcAff{2}City University of Hong Kong\quad
  \qcAff{3}University of Texas at Austin}
\qcContact{\href{mailto:xiaoyulics@gmail.com}{xiaoyulics@gmail.com}}
\date{}
\qcFootnoteLeft{Code and data: \href{https://github.com/xiaoyulics/singularity-distribution}{\texttt{github.com/xiaoyulics/singularity-distribution}} (\cref{app:numerics}).}

\begin{document}
\maketitle

\begin{abstract}
\noindent Automating AI research creates a positive feedback loop between software capability and research effort. We study a stochastic realisation of this loop in which independent, positive, multiplicative improvements arrive at a rate proportional to a power of current capability. Within this model, randomness leaves the deterministic explosion threshold unchanged but replaces the blow-up date by an explicit distribution. For equal-sized improvements the explosion time is a geometrically weighted sum of independent exponentials; its $q$-series representation gives its transforms, moments and tail behaviour, including a two-term small-time asymptotic. For Lomax-distributed improvement sizes the explosion time is exactly Gamma distributed. Self-similarity implies that, with parameters known, the relative uncertainty about the remaining time does not vanish as explosion approaches. We formulate a fixed-sample test between constant and geometrically increasing idea-arrival rates. If the baseline rate and the alternative log-rate increment $\ell$ are known, the minimum sample size at error level $\alpha$ is asymptotic to $(12z_\alpha^2)^{1/3}\ell^{-2/3}$, where $z_\alpha$ is the upper-$\alpha$ standard normal quantile. Uncertainty about the growth exponent near its critical value produces a $1/t$ tail with an explicit constant, and hence an infinite conditional mean waiting time, although explosion is certain given a positive exponent. Separate extensions show that a fixed deployment delay prevents explosion in the specified delayed-feedback models, that power-law diversion of effort shifts the growth threshold, and that in a deterministic oversight race a fixed allocation share does not change the asymptotic regime, except at a critical equality of exponents. Numerical illustrations show how timing and detectability depend on the granularity of improvements. They are not empirical forecasts of AI development.
\end{abstract}

\qcMeta{Keywords}{intelligence explosion; pure birth process; explosion time; $q$-series; exponential functionals; perpetuities; hypothesis testing; delay differential equations}
\qcMeta{MSC 2020}{60J27, 60G51, 33D15, 62F03, 34K25, 91B62}

\vspace{4pt}
\noindent\begin{minipage}{\textwidth}
\includegraphics[width=\textwidth]{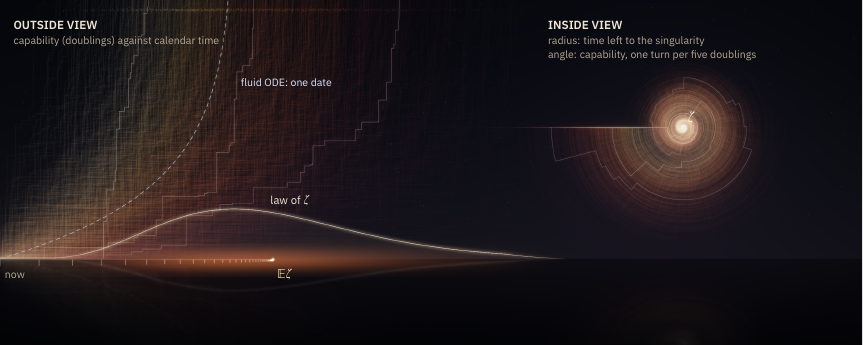}
\captionof{figure}{\emph{The singularity has a distribution.} \textbf{Left:} 1\,400 simulated histories, coloured by explosion time, rise as staircases of ideas and leave the sky at their own explosion times; the fluid ODE (dashed) has a single date. The horizon glows with the exact law of $\zeta$ (here Gamma, \cref{cor:gamma}), and its ticks mark $\E T_n\to\E\zeta$. \textbf{Right:} the same histories with radius $\zeta-t$ and angle $\propto\log A_t$, each a stepped logarithmic spiral into the singularity (\cref{sec:midnight}). Details in \cref{app:numerics}.}
\label{fig:hero}
\end{minipage}

\vspace{6pt}
{\setcounter{tocdepth}{1}\small\tableofcontents}

\section{Introduction}\label{sec:intro}

\subsection{A feedback law and what it leaves out}

In 1960, von Foerster, Mora and Amiot fitted a hyperbola $N(t)\propto(t^\ast-t)^{-1}$ to two millennia of world population and dated its singularity to \emph{Friday, 13 November, A.D.\ 2026} \cite{vFMA60}. That date is now a month away, and the world population remains finite. Hyperbolic growth is back all the same, this time as a hypothesis about software.

Chan et al.\ \cite{CASP26} argue that the automation of AI research and development could trigger an \emph{intelligence explosion}, in which AI progress, driven by AI itself, speeds up so much that years of advances happen within months. The mechanism is a feedback loop: AI systems enlarge the effective research workforce, and that workforce produces better AI systems. The quantitative core of their supplementary materials is the idea production function of semi-endogenous growth theory \cite{Jon95,BJVW20}, $\dot A=A^{1-\beta}E^{\lambda}$. Here $A$ is software quality, $E$ is effective research labour, $\lambda$ measures returns to scale in labour and $\beta$ measures how fast ideas get harder to find. Under full automation $E=kA$, and
\begin{equation}\label{eq:ode}
\dot A = c\,A^{1+\eps},\qquad \eps:=\lambda-\beta=\beta(r-1),\qquad r:=\lambda/\beta .
\end{equation}
When $r>1$, \eqref{eq:ode} reaches infinity at the finite time $T^\ast=(\eps cA_0^{\eps})^{-1}$. The authors read this ``infinite progress in finite time'' as the point where the model breaks down, not as a prediction \cite[SM]{CASP26}. Their central estimates are $\lambda=1.40$ and $\beta=1.01$, so that $\eps=\NumEps$ and $r=\lambda/\beta=\NumR$; we use these values for illustration only.

The deterministic law cannot say how uncertain the blow-up date is, how early an observer could tell acceleration from steady exponential growth, or what a deployment lag does to the singularity. These are the questions of this paper. We keep the drift of \eqref{eq:ode}, replace its continuum of progress by discrete, randomly timed improvements, and ask what can be computed exactly.

\subsection{AI-generated mathematics and the research feedback loop}\label{sec:aimath}

Recent developments in AI-assisted mathematical research make the feedback mechanism in \eqref{eq:ode} more concrete, while showing how difficult it is to measure. In September 2026, OpenAI reported an AI-generated analytical argument, with a Lean formalisation, for finite-time singularity formation in the three-dimensional Navier--Stokes equations with a smooth external force \cite{OAI-NS}. The Clay Mathematics Institute described the problem as apparently settled, while noting that its evaluation and credit-assignment procedures would take time \cite{CMI-NS}. The forcing assumption matters: this is not a claimed resolution of the corresponding unforced regularity question. It is also a mathematical singularity of a different kind from the one studied here.

On 6 October 2026, OpenAI released a repository of mathematical manuscripts produced by an internal model \cite{OAI-math,OAI-repo}. As released it held 722 manuscripts grouped into 372 families of related results, with computer-checkable Lean formalisations for many, but not all, of them; its maintainers caution that some unformalised results could have issues \cite{OAI-repo}. A day later three manuscripts were withdrawn after a sign error was found, and the repository now lists 719 \cite{OAI-repo}. These counts are of manuscripts, not of independent, externally certified breakthroughs, and the withdrawal shows that validation continues after release. Such distinctions matter whenever a ``new idea'' is treated as an observable event.

The connection to AI research automation is more direct, but still incomplete. In September 2026, OpenAI reported increasing use of coding agents within its research organisation: by mid-August it measured 3.1 agent-workdays of runtime per human workday, together with more experiment activity and delegation of more complex tasks \cite{OAI-RA}. These are operational indicators, not estimates of effective research labour $E$, of the returns ratio $r=\lambda/\beta$, or of the exponent $\eps=\beta(r-1)$. The same report notes confounding from increased compute and the continued importance of human steering and bottlenecks.

These observations motivate a mathematical study of research acceleration; they do not establish that a supercritical feedback loop exists. They also sharpen the distinction between producing a candidate result and validating, understanding and incorporating it. The ideal observer of \cref{sec:detect} sees every effective idea and its arrival time, whereas actual observers see manuscripts, experiments, model releases and incomplete verification records. This gap is a reason to study statistical detectability, deployment lags and the limits of extrapolation, and not to read the calibrated explosion dates below as empirical forecasts.

\subsection{Model, assumptions and contributions}

\paragraph{The model.} In state $A$, the next improvement (an ``idea'') arrives after an exponential time with rate $\kappa A^{\eps}$ and multiplies $A$ by $1+H$, where the sizes $H$ are i.i.d.\ and strictly positive. We call this a \emph{Good process}, after I.\,J.\ Good's ultraintelligent machine \cite{Goo66}. When all ideas have the same size $h$, the idea count is a pure birth process with \emph{geometric} rates $q_n=q_0\rho^n$, $\rho=(1+h)^{\eps}$, and its explosion time is
\[
\zeta=\sum_{n\ge0}\frac{e_n}{q_0\rho^n},\qquad e_n\ \text{i.i.d.\ }\Exp(1).
\]
Throughout, for $\eps>0$, $q:=1/\rho\in(0,1)$, and $\ell:=\log\rho=\eps\log(1+h)$ is the increment of the log-rate per idea. The model is one tractable stochastic realisation of \eqref{eq:ode}, not the unique one, and \cref{rem:scope} lists which assumptions support which result.

\paragraph{Results.} Write $\zb:=q_0\zeta$ for the standardised explosion time.
\begin{enumerate}[leftmargin=*,label=\textbf{\arabic*.}]
\item \textbf{Explosion criterion} (\cref{thm:dichotomy}). For every i.i.d.\ positive idea-size law in this model, explosion is almost sure if $r>1$ and impossible if $r\le1$. For equal sizes and a monotone rate function, explosion of the chain is equivalent to blow-up of its fluid equation.
\item \textbf{Exact timing law} (\cref{thm:law,thm:random,cor:gamma}). For equal sizes, $\zb$ has Laplace transform $1/(-s;q)_\infty$, moments $n!/(q;q)_n$, Mellin transform $\Gamma(1+s)(q^{1+s};q)_\infty/(q;q)_\infty$ and an alternating $q$-series distribution function. Its slow tail is $\P(\zb>t)\sim e^{-t}\sum_mp(m)q^m$, and at fixed $q$ it has the two-term small-time asymptotic $\log\P(\zb\le t)=-\frac{L^2}{2\ell}-\frac{L\log L}{\ell}+O(L)$, $L=\log(1/t)$. For random sizes $\zb$ is a perpetuity with explicit moments, and when $\log(1+H)$ is exponential (Lomax sizes), $\zb$ is exactly Gamma distributed. Most of these formulas are known (\cref{tab:contrib}); we give self-contained proofs.
\item \textbf{Residual-time self-similarity} (\cref{thm:selfsim,prop:logperiodic,prop:uncertain}). With parameters fixed and known, at every stopping time $\tau$ the remaining time $\zeta-\tau$, multiplied by the current idea rate, has the law of $\zb$, independent of the past. Relative predictive uncertainty therefore does not vanish as explosion approaches. In coordinates centred on the realised explosion time, the mean density of ideas is log-periodic with harmonics $\Gamma(1-2\pi ik/\ell)$. If the growth exponent is uncertain with positive density at the critical value, then $t\,\P(\zeta>t\mid\eps>0)$ converges to an explicit constant.
\item \textbf{Fixed-sample detection} (\cref{thm:detect,cor:warning}). For the binary experiment with known baseline rate and known increment $\ell$, decided once after the first $n$ waiting times, the minimum sample size at level $\alpha$ satisfies $n^\ast\sim(12z_\alpha^2)^{1/3}\ell^{-2/3}$, with non-asymptotic bounds on both sides.
\item \textbf{Model extensions} (\cref{thm:lag,prop:diversion,thm:race}). In a deterministic delayed-feedback model any fixed delay $D>0$ prevents blow-up and gives $\log\log A(t)\sim t\log r/D$; the corresponding stochastic model does not explode. Power-law diversion of effort changes the explosion threshold. In a deterministic race with a fixed allocation share, the asymptotic regime of the oversight-to-capability ratio is determined by exponents, except at a critical equality where the allocation enters the limit.
\end{enumerate}

\Cref{tab:contrib} separates what is classical from what may be new.

\begin{table}[t]
\caption{Contribution map. ``Known'' means that the formula appears in the cited literature; we give self-contained proofs because the application reading is the point. Novelty claims are limited to what we checked (\cref{sec:related}).}
\label{tab:contrib}
\centering\small
\renewcommand{\arraystretch}{1.16}
\begin{tabular}{@{}>{\raggedright\arraybackslash}p{0.40\textwidth}>{\raggedright\arraybackslash}p{0.24\textwidth}>{\raggedright\arraybackslash}p{0.30\textwidth}@{}}
\toprule
\textsf{Component} & \textsf{Where} & \textsf{Status}\\
\midrule
Explosion criterion; exponential-functional representation & Thm~\ref{thm:dichotomy}; Rem.~\ref{rem:lamperti} & classical structure applied to \eqref{eq:ode}\\
Laplace transform, density, integer and fractional moments; random-size moments & Thm~\ref{thm:law}\ref{it:laplace}--\ref{it:series},\,\ref{it:mellin}; Thm~\ref{thm:random} & known \cite{CPY97,DGR02,BBY04,GRZ04}\\
Slow tail from the first exponential mode & \cref{thm:law}\ref{it:slow} & explicit consequence of the known density\\
Two-term small-time asymptotic & \cref{thm:law}\ref{it:fast} & known in sharper form \cite{Aur08,BR08}; self-contained proof\\
Lomax sizes give a Gamma explosion time & \cref{cor:gamma} & classical \cite{Ver79,BDIM18}; derived from \cref{thm:random}\\
Residual-time self-similarity as a forecasting statement & \cref{thm:selfsim} & direct consequence of self-similarity\\
$\Gamma$-coefficients of the explosion-centred intensity & \cref{prop:logperiodic} & exact specialisation\\
Exact $1/t$ tail constant under parameter uncertainty & \cref{prop:uncertain} & routine Abelian argument; not found stated for this model\\
Sharp fixed-sample detection threshold & \cref{thm:detect} & statistical centrepiece of the paper\\
Delay, diversion and oversight race & \cref{sec:brakes} & separate model extensions with their own scopes\\
\bottomrule
\end{tabular}
\end{table}

\subsection{Related work and organisation}\label{sec:related}

Finite-time singularities in growth models go back at least to \cite{vFMA60}. In economics they appear in Kremer's model of population and technology \cite{Kre93} and in discussions of AI and growth \cite{AJJ19}. The modelling of a software intelligence explosion that we start from is surveyed in \cite{CASP26}, which builds on estimates of $r$ by Ho and Whitfill \cite{HW25}, on the feedback-loop analysis of Davidson et al.\ \cite{DHHK26}, and on Ord's discussion of the time between rounds of improvement \cite{Ord26}.

On the probabilistic side, the explosion time of a Good process is the exponential functional of a compound Poisson subordinator in Lamperti's representation \cite{Lam72}, and such functionals are well studied \cite{CPY97,BY05}. The Poisson case, which is our lattice case, arose in the analysis of AIMD congestion-control algorithms \cite{DGR02,GRZ04}. Guillemin, Robert and Zwart \cite{GRZ04} give its density as an alternating series, its Laplace transform as a $q$-series, and the moment formulas of \cref{thm:law}\ref{it:moments} and \cref{thm:random}. Bertoin, Biane and Yor \cite{BBY04} connected it to $q$-integrals and the log-normal moment problem and obtained its fractional moments, which give the Mellin formula \ref{it:mellin}. Perpetuities of the form in \cref{thm:random} go back at least to Vervaat \cite{Ver79}. The Gamma example of \cref{cor:gamma} is classical; it is stated explicitly, with a reference to \cite{Ver79}, in \cite[Example~2.1]{BDIM18}. The small-time asymptotics of geometrically weighted sums of i.i.d.\ positive variables are known in sharper form. Aurzada \cite{Aur08} gives the leading term, and Borovkov and Ruzankin \cite[Theorem~1.1]{BR08} give the second term together with an explicit $O(L)$ constant. \Cref{thm:law}\ref{it:fast} is a weaker form of their result, included with a short self-contained proof. We did not find \cref{prop:uncertain} stated for this model; it follows by a routine Abelian-type argument. We are not aware of earlier treatments of the residual-time invariance as a forecasting statement, of the $\Gamma$-coefficients of the explosion-centred intensity, or of the detection, delay and race results in this context. These statements are limited to the sources we checked.

\paragraph{Organisation.} \Cref{sec:model} defines Good processes, states the modelling assumptions and fixes the calibration. \Cref{sec:law} proves the explosion criterion and the exact timing law. \Cref{sec:midnight} treats forecasting from inside the explosion and parameter uncertainty, and \cref{sec:detect} the fixed-sample detection problem. \Cref{sec:brakes} collects the separate model extensions. \Cref{sec:calibration} gives numerical illustrations, and \cref{sec:open} lists limitations and open problems. \Cref{app:numerics} specifies the numerical methods and how to reproduce every number.

\section{The model and its calibration}\label{sec:model}

\begin{definition}[Good process]\label{def:good}
Let $A_0>0$, $\kappa>0$, $\eps\in\R$, and let $\mu$ be a probability measure on $(0,\infty)$, the \emph{idea-size law}. Let $(e_n)_{n\ge0}$ be i.i.d.\ $\Exp(1)$ and $(H_n)_{n\ge1}$ i.i.d.\ with law $\mu$, all independent. Put
\[
A^{(0)}:=A_0,\qquad A^{(n)}:=A^{(n-1)}(1+H_n),\qquad \tau_n:=\frac{e_n}{\kappa\,(A^{(n)})^{\eps}},\qquad T_n:=\sum_{k<n}\tau_k,
\]
and $\zeta:=\lim_n T_n\in(0,\infty]$. The \emph{Good process} is $A_t:=A^{(n)}$ for $T_n\le t<T_{n+1}$, with $A_t:=\infty$ for $t\ge\zeta$. Its idea count is $N_t:=n$ on $[T_n,T_{n+1})$, $(\F_t)$ is the usual augmentation of the filtration generated by $(A_t)$, and $\zeta$ is its \emph{explosion time} ($\zeta=\infty$ means no explosion).
\end{definition}

In state $a$ the process waits an exponential time with rate $\kappa a^{\eps}$ and then jumps to $a(1+H)$. It is a pure-jump Markov process, minimal in the sense that nothing is defined after the first explosion. We call $\zeta$ the explosion time or the singularity time of the model; it is not a separately defined capability threshold.

\begin{remark}[Scope of the model]\label{rem:scope}
We study a tractable stochastic realisation of the feedback law \eqref{eq:ode}, not a unique stochastic model implied by that law. The assumptions of independent positive multiplicative improvements and exponential holding times make the clock explicit. Our robustness statements concern the specified family of processes; they do not establish robustness to arbitrary dependence, negative shocks, time-varying returns or alternative observation schemes. The assumptions are used as follows.
\begin{itemize}[leftmargin=*]
\item Geometric rates together with exponential holding times produce the $q$-products of \cref{thm:law}.
\item The memoryless property of the holding times is what makes the residual-time statement of \cref{thm:selfsim} hold at \emph{arbitrary} stopping times, not just at idea times.
\item The explosion criterion of \cref{thm:dichotomy}\ref{it:dich-power} needs i.i.d.\ strictly positive sizes but no moment condition on $H$; matching a finite level drift through $\E H$ (\cref{box:calibration}, convention A) does need $\E H<\infty$.
\item What counts as one ``idea'' is a modelling choice. Grouping several exponential waiting stages into one idea generally produces a non-exponential holding time, so a coarser idea definition is not automatically another Good process.
\end{itemize}
\end{remark}

\paragraph{The lattice case.} When $H\equiv h$ we write
\[
q_n:=\kappa A_0^{\eps}(1+h)^{\eps n}=q_0\rho^n,\qquad \rho:=(1+h)^{\eps},\qquad q:=\rho^{-1},\qquad \ell:=\log\rho=\eps\log(1+h).
\]
The waiting time in state $n$ is $e_n/q_n$, and
\begin{equation}\label{eq:zeta}
\zeta=\sum_{n\ge0}\frac{e_n}{q_0\rho^n},\qquad \zb:=q_0\zeta=\sum_{n\ge0}e_nq^n .
\end{equation}
The law of $\zb$ depends on $q$ alone; we denote it $\Lambda_q$. Thus $\zb$ measures time in units of the current expected waiting time $1/q_0$.

\begin{remark}[Lamperti representation]\label{rem:lamperti}
Let $M$ be a unit-rate Poisson process independent of $(H_i)$, and let $\xi_s:=\sum_{i\le M_s}\log(1+H_i)$, a compound Poisson subordinator. Measure time in ``ideas'' $s$ and convert it to calendar time through $\dd t=\dd s/(\kappa A^{\eps})$. Then $A_t=A_0\exp(\xi_{\sigma(t)})$, where $\sigma$ is the inverse of $s\mapsto(\kappa A_0^{\eps})^{-1}\int_0^s e^{-\eps\xi_u}\dd u$, and
\[
\zeta=(\kappa A_0^{\eps})^{-1}\int_0^\infty e^{-\eps\xi_s}\dd s .
\]
Thus $A$ is a positive self-similar Markov process in the sense of Lamperti \cite{Lam72}, with a negative index: if $A$ starts at $a$, then $(cA_{c^{\eps}t})_{t\ge0}$ has the law of $A$ started at $ca$. Moreover $\zeta$ is an \emph{exponential functional} of the subordinator $\eps\xi$ \cite{CPY97,BY05}.
\end{remark}

\subsection{Calibration: what is held fixed}\label{sec:calib}

A drift law does not determine a jump process, and the phrase ``doubling time'' has at least three meanings for one. \Cref{box:calibration} separates them. All calibrated numbers in this paper use convention~B.

\begin{qcpanel}[\label{box:calibration}]{Three meanings of \emph{doubling time}}
Write $\mathcal L$ for the generator of the lattice Good process: $\mathcal Lf(a)=\kappa a^{\eps}\big(f(a(1+h))-f(a)\big)$.

\smallskip\noindent\textbf{A. Local level drift.} $\mathcal L\operatorname{id}(a)=\kappa h\,a^{1+\eps}$. Matching \eqref{eq:ode} requires $\kappa=c/h$, so $q_0=cA_0^{\eps}/h$. (For random sizes, $\kappa\,\E H$ replaces $\kappa h$, which needs $\E H<\infty$.)

\smallskip\noindent\textbf{B. Local log drift} (used throughout). $g:=\mathcal L\log(A_0)=q_0\log(1+h)$. We set $g=\log2/d_0$ with $d_0=4.5$ months, so $g=\NumG$ per month and
\[
q_0=\frac{g}{\log(1+h)} .
\]
Here $d_0$ is a \emph{log-drift-equivalent doubling-time proxy}; it is not in general the expected time to reach $2A_0$. We say ``local'' because unconditional expectations such as $\E\log A_t$ are infinite for every $t>0$ once explosion before $t$ has positive probability.

\smallskip\noindent\textbf{C. First-passage doubling.} For the ODE with $cA_0^{\eps}=g_0$, the actual time to reach $2A_0$ is $(1-2^{-\eps})/(\eps g_0)$, which equals $\log2/g_0$ only in the limit $\eps\to0$. For the lattice process,
\[
\E T(2A_0)=\frac{1-q^{m}}{q_0(1-q)},\qquad m=\Big\lceil\frac{\log 2}{\log(1+h)}\Big\rceil .
\]

\smallskip\noindent\textbf{Numbers at $\eps=\NumEps$, $d_0=4.5$.} With $g_0=g$, the ODE blows up at $T^\ast=1/(\eps g)=\NumTfluid$ months, and its actual first doubling takes $\NumODEdouble$ months. If instead the ODE's actual first doubling is fixed at $4.5$ months, it blows up at $\NumTfluidC$ months. Under convention~B the mean first doubling of the lattice process is $\NumFirstDblA$, $\NumFirstDblAb$, $\NumFirstDblB$, $\NumFirstDblC$ and $\NumFirstDblD$ months for $h=0.01,0.05,0.1,1,9$. Only $h=1$ gives exactly $d_0$; the $h=9$ process overshoots $2A_0$ at its first tenfold jump, so its first-doubling time is the whole initial wait.

\smallskip\noindent\textbf{A versus B.} If the fluid coefficient $cA_0^{\eps}=g$ is held fixed, A gives $q_0=g/h$ and B gives $q_0=g/\log(1+h)$; they differ for every $h>0$. If instead $cA_0^{\eps}=1/d_0$, they coincide at $h=1$, but then the level-drift ODE blows up at $d_0/\eps=\NumTlevel$ months.
\end{qcpanel}

The ODE \eqref{eq:ode} is the small-idea limit of Good processes in the following sense.

\begin{proposition}[Fluid limit of hitting times]\label{prop:fluid}
Consider the lattice case with $\kappa=c/h$ (convention A) for fixed $c>0$, $\eps\in\R$ and $A_0>0$. For $a>A_0$ let $T^h(a):=\inf\{t:A_t\ge a\}$ and $t(a):=\int_{A_0}^a \dd x/(cx^{1+\eps})$, the time the solution of \eqref{eq:ode} takes to reach $a$. Then $T^h(a)\to t(a)$ in $L^2$ as $h\downarrow0$. If $\eps>0$, then also $\zeta^h\to T^\ast=(\eps cA_0^{\eps})^{-1}$ in $L^2$.
\end{proposition}

\begin{proof}
Let $x_k:=A_0(1+h)^k$ and $n:=\min\{k:x_k\ge a\}$, so that $T^h(a)=T_n=\sum_{k<n}e_k/q_k$ with $q_k=c\,x_k^{\eps}/h$. Since $x_{k+1}-x_k=hx_k$,
\[
\E T_n=\sum_{k<n}\frac{h}{c\,x_k^{\eps}}=\sum_{k<n}\frac{x_{k+1}-x_k}{c\,x_k^{1+\eps}}
\]
is a Riemann sum for $\int_{A_0}^{x_n}\dd x/(cx^{1+\eps})$ with mesh at most $ha(1+h)$. Since $x_n\in[a,a(1+h))$, we get $\E T_n\to t(a)$. Moreover $\Var T_n=\sum_{k<n}q_k^{-2}\le(\max_{k<n}q_k^{-1})\,\E T_n$, and $\max_{k<n}q_k^{-1}\le h\max(A_0^{-\eps},a^{-\eps})/c\to0$. Hence $\E|T_n-t(a)|^2=\Var T_n+(\E T_n-t(a))^2\to0$. For $\eps>0$, \eqref{eq:zeta} gives
\[
\E\zeta^h=\frac{h}{cA_0^{\eps}(1-(1+h)^{-\eps})}\to\frac1{\eps cA_0^\eps},\qquad
\Var\zeta^h=\frac{h^2}{c^2A_0^{2\eps}(1-(1+h)^{-2\eps})}\to0. \qedhere
\]
\end{proof}

The proposition concerns hitting times; it is not a path-space limit theorem. Under convention~B the same limit holds, because $h/\log(1+h)\to1$.

\begin{figure}[t]
\centering
\includegraphics[width=\textwidth]{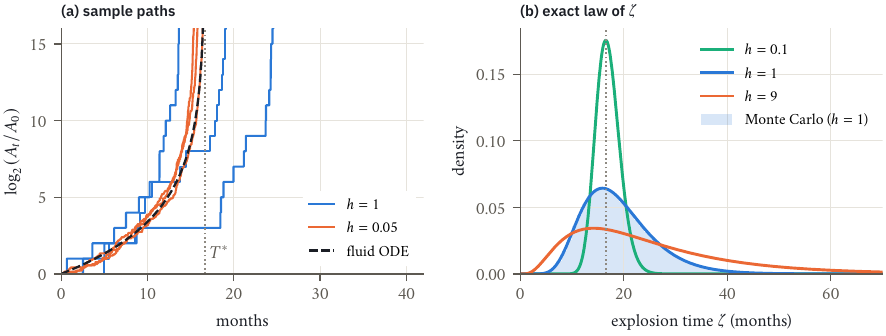}
\caption{(a) Sample paths of $\log_2 A_t$ at $\eps=\NumEps$ under convention~B (local log drift $g=\log2/4.5$ per month held fixed), for doubling-sized ideas ($h=1$, whose first doubling is expected after exactly $4.5$ months) and small ideas ($h=0.05$, whose first doubling is expected after $\NumFirstDblAb$ months). The dashed curve is the fluid ODE with $cA_0^{\eps}=g$, the small-idea limit, which blows up at $T^\ast=\NumTfluid$ months. (b) Exact density of the explosion time (\cref{thm:law}\ref{it:series}) for three idea sizes under convention~B, with a Monte Carlo histogram for $h=1$.}
\label{fig:paths}
\end{figure}

\section{Explosion criterion and exact timing law}\label{sec:law}

\begin{theorem}[Explosion criterion]\label{thm:dichotomy}
\leavevmode
\begin{enumerate}[label=(\alph*)]
\item\label{it:dich-general} \emph{Lattice sizes, arbitrary rates.} Let $h>0$, let $\nu:(0,\infty)\to(0,\infty)$ be any function, and consider the pure birth chain with rates $q_n=\nu(A_0(1+h)^n)$, i.e.\ the lattice process with rate function $\nu$ in place of $\kappa a^{\eps}$. Then $\P(\zeta<\infty)=1$ if $\sum_n 1/q_n<\infty$, and $\P(\zeta=\infty)=1$ otherwise. If $\nu$ is monotone, then $\sum_n1/q_n<\infty$ if and only if
\[
\int_{A_0}^\infty\frac{\dd x}{x\,\nu(x)}<\infty,
\]
which is the Osgood condition for finite-time blow-up of the fluid equation $\dot a=h\,a\,\nu(a)$.
\item\label{it:dich-power} \emph{Power rates, arbitrary sizes.} For the Good process with any i.i.d.\ strictly positive idea-size law $\mu$: if $\eps>0$ then $\E[\zeta^m]<\infty$ for every $m\ge1$, so in particular $\zeta<\infty$ a.s.; if $\eps\le0$ then $\zeta=\infty$ a.s. No moment condition on $\mu$ is needed.
\end{enumerate}
\end{theorem}

\begin{proof}
\ref{it:dich-general} Here $\zeta=\sum_n e_n/q_n$ with deterministic $q_n$. If $\sum 1/q_n<\infty$ then $\E\zeta<\infty$. If $\sum 1/q_n=\infty$, then $\sum_n\log(1+1/q_n)=\infty$ (use $\log(1+x)\ge x/2$ for $x\le1$), so $\E e^{-\zeta}=\prod_n(1+1/q_n)^{-1}=0$ and $\zeta=\infty$ a.s. For the integral test, suppose $\nu$ is nondecreasing. Then $u\mapsto1/\nu(A_0(1+h)^u)$ is nonincreasing, so its sum over $u\in\N$ and its integral over $[0,\infty)$ converge together. The substitution $x=A_0(1+h)^u$ turns the integral into $(\log(1+h))^{-1}\int_{A_0}^\infty\dd x/(x\nu(x))$. If $\nu$ is nonincreasing, both the sum and the integral diverge.

\ref{it:dich-power} Let $Y_i:=(1+H_i)^{-\eps}$. Then $\zb:=\kappa A_0^\eps\zeta=\sum_n e_n\prod_{i\le n}Y_i$. If $\eps>0$, then $Y_i\in(0,1)$ a.s., so $y_m:=\E Y^m<1$ for every $m$. By Minkowski's inequality in $L^m$ and independence,
\[
\|\zb\|_m\le\sum_{n\ge0}\|e_n\|_m\Big\|\prod_{i\le n}Y_i\Big\|_m=(m!)^{1/m}\sum_{n\ge0}y_m^{n/m}<\infty .
\]
If $\eps\le0$, then $Y_i\ge1$ and $\zb\ge\sum_ne_n=\infty$ a.s.
\end{proof}

\begin{corollary}[The threshold is $r=1$]\label{cor:threshold}
In the parametrisation of \eqref{eq:ode}, with $\beta>0$, a Good process explodes almost surely if $r>1$ and never explodes if $r\le1$, for every i.i.d.\ strictly positive idea-size law and every $\kappa$. Within this model, randomness changes when the explosion happens, but not whether it does.
\end{corollary}

\begin{remark}[Saturation]\label{rem:saturation}
Part \ref{it:dich-general} also covers the working paper's observation that $r$ must eventually fall below $1$ because of physical limits \cite[note~5]{CASP26}. Any monotone rate function with $\int^\infty\dd x/(x\nu(x))=\infty$, for example one that saturates, prevents explosion almost surely. The singularity is a property of the idealised power law.
\end{remark}

We use the $q$-Pochhammer symbols $(z;q)_n:=\prod_{j=0}^{n-1}(1-zq^j)$ and $(z;q)_\infty:=\prod_{j\ge0}(1-zq^j)$, and Euler's two identities \cite[Ch.~2]{And76}, \cite[\S1.3]{GR04}:
\begin{align}
\sum_{n\ge0}\frac{z^n}{(q;q)_n}&=\frac{1}{(z;q)_\infty}\quad(|z|<1),\label{eq:euler1}\\
\sum_{n\ge0}\frac{(-1)^nq^{n(n-1)/2}z^n}{(q;q)_n}&=(z;q)_\infty\quad(z\in\C).\label{eq:euler2}
\end{align}
Recall also that $1/(q;q)_n=\sum_{m\ge0}p_{\le n}(m)q^m$, where $p_{\le n}(m)$ is the number of partitions of $m$ into parts of size at most $n$, and that $1/(q;q)_\infty=\sum_m p(m)q^m=:\mathcal P(q)$ is the partition generating function \cite[Ch.~1]{And76}.

\begin{theorem}[Exact law, lattice case]\label{thm:law}
Let $q\in(0,1)$, $\ell=\log(1/q)$, and $\zb=\sum_{n\ge0}e_nq^n$ as in \eqref{eq:zeta}.
\begin{enumerate}[label=(\roman*)]
\item\label{it:laplace} \emph{Laplace transform.} $\E e^{-s\zb}=1/(-s;q)_\infty$ for $s>-1$.
\item\label{it:moments} \emph{Moments count partitions.} For $n\in\N$,
\[
\E\zb^{\,n}=\frac{n!}{(q;q)_n},\qquad\text{equivalently}\qquad \frac{\E\zb^{\,n}}{n!}=\sum_{m\ge0}p_{\le n}(m)\,q^m .
\]
\item\label{it:series} \emph{Distribution.} For $t\ge0$,
\[
\P(\zb>t)=\frac{1}{(q;q)_\infty}\sum_{k\ge0}\frac{(-1)^kq^{k(k+1)/2}}{(q;q)_k}\,e^{-t/q^k},
\qquad
f_{\zb}(t)=\frac{1}{(q;q)_\infty}\sum_{k\ge0}\frac{(-1)^kq^{k(k-1)/2}}{(q;q)_k}\,e^{-t/q^k}.
\]
\item\label{it:slow} \emph{Slow tail.} For $t\ge0$,
\[
\Big|\P(\zb>t)-\mathcal P(q)\,e^{-t}\Big|\le\frac{(-q;q)_\infty}{(q;q)_\infty}\,e^{-t/q} .
\]
\item\label{it:fast} \emph{Small-time asymptotic.} At fixed $q$, as $t\downarrow0$, with $L:=\log(1/t)$,
\[
\log\P(\zb\le t)=-\frac{L^2}{2\ell}-\frac{L\log L}{\ell}+O(L).
\]
\item\label{it:mellin} \emph{Mellin transform.} The function $s\mapsto\E\zb^{\,s}$ is entire, and for all $s\in\C$
\[
\E\zb^{\,s}=\frac{\Gamma(1+s)\,(q^{1+s};q)_\infty}{(q;q)_\infty},
\]
where the right-hand side has removable singularities at $s=-1,-2,\dots$. In particular, for $n\ge1$,
\[
\E\big[\zb^{-n}\big]=\frac{\ell}{(n-1)!}\prod_{m=1}^{n-1}(\rho^m-1),\qquad\text{so}\qquad \E[1/\zb]=\ell .
\]
\end{enumerate}
\end{theorem}

The leading term in \ref{it:fast} is of log-normal type, but the law is not log-normal: the $L\log L$ term is part of the statement. The asymptotic is at fixed $q$ and is not claimed to be uniform in the separate small-idea limit $q\uparrow1$. Item \ref{it:fast} is a weaker form of a known result. For exponential summands, Borovkov and Ruzankin \cite[Theorem~1.1]{BR08} prove
\[
\log\P(\zb\le t)=-\frac{L^2}{2\ell}-\frac{L\log L}{\ell}-\frac{(\ell-2\log\ell-2)\,L}{2\ell}+o(L),
\]
and our high-precision evaluations at $q=\frac12$ are consistent with this (\cref{app:numerics}).

\begin{proof}
\ref{it:laplace} By independence, $\E e^{-s\zb}=\prod_k\E e^{-sq^ke_k}=\prod_k(1+sq^k)^{-1}$ for $s>-1$.

\ref{it:moments} By \eqref{eq:euler1} with $z=-s$, $\E e^{-s\zb}=\sum_n(-s)^n/(q;q)_n$ for $|s|<1$. Since $\zb$ has exponential moments of small order, this expansion equals $\sum_n(-s)^n\E\zb^{\,n}/n!$. The partition statement is the expansion $\prod_{j=1}^n(1-q^j)^{-1}=\prod_{j=1}^n\sum_{a_j\ge0}q^{ja_j}$.

\ref{it:series} Let $\zb_n:=\sum_{k<n}e_kq^k$. Its summands are independent exponentials with distinct rates $q^{-k}$, so by the classical formula for the hypoexponential law,
\[
\P(\zb_n>t)=\sum_{k<n}c^{(n)}_k e^{-t/q^k},\qquad c^{(n)}_k:=\prod_{\substack{j<n\\ j\neq k}}\frac{1}{1-q^{j-k}} .
\]
The factors with $j<k$ equal $(1-q^{-m})^{-1}=-q^m/(1-q^m)$ with $m=k-j\in\{1,\dots,k\}$, and those with $j>k$ equal $(1-q^m)^{-1}$ with $m=j-k\in\{1,\dots,n-1-k\}$. Hence
\[
c^{(n)}_k=\frac{(-1)^kq^{k(k+1)/2}}{(q;q)_k\,(q;q)_{n-1-k}}\xrightarrow[n\to\infty]{}\frac{(-1)^kq^{k(k+1)/2}}{(q;q)_k\,(q;q)_\infty},
\]
and $|c^{(n)}_k|\le q^{k(k+1)/2}/(q;q)_\infty^2$, which is summable. Since $\zb_n\uparrow\zb$, $\P(\zb_n>t)\uparrow\P(\zb>t)$, and dominated convergence gives the series. As a check, at $t=0$ the series equals $1$ by \eqref{eq:euler2} with $z=q$. Differentiating term by term gives the density; this is justified because, uniformly in $t\ge0$, the differentiated series is dominated by the convergent series $\sum_kq^{k(k-1)/2}/(q;q)_k$.

\ref{it:slow} The $k=0$ term of the series is $e^{-t}/(q;q)_\infty$. For $k\ge1$ we have $e^{-t/q^k}\le e^{-t/q}$, and by \eqref{eq:euler2} with $z=-q$,
\[
\sum_{k\ge1}\frac{q^{k(k+1)/2}}{(q;q)_k}\le(-q;q)_\infty .
\]
(Probabilistically, $\zb=e_0+X$ with $X$ independent of $e_0$ and $\E e^{X}=\prod_{k\ge1}(1-q^k)^{-1}=\mathcal P(q)$: a slow takeoff needs only a slow \emph{first} idea.)

\ref{it:fast} \emph{Upper bound.} For $s\ge1$ let $m:=\lfloor\log s/\ell\rfloor$ and write $\log s=\ell(m+\theta)$ with $\theta\in[0,1)$. Since $\log(1+x)\ge\log x$ and $sq^k\ge1$ for $k\le m$,
\[
\log(-s;q)_\infty\ \ge\ \sum_{k=0}^m(\log s-k\ell)=\ell(m+1)\Big(\frac m2+\theta\Big)\ \ge\ \frac{\ell(m+\theta)^2}{2}=\frac{(\log s)^2}{2\ell},
\]
because $(m+1)(m+2\theta)-(m+\theta)^2=m+2\theta-\theta^2\ge0$. By Chernoff's bound and \ref{it:laplace}, $\log\P(\zb\le t)\le st-(\log s)^2/(2\ell)$. Choose $\log s=L+\log(L/\ell)$, which is admissible once $L\ge\ell$, so that $st=L/\ell$. Then
\[
\log\P(\zb\le t)\le\frac L\ell-\frac{(L+\log(L/\ell))^2}{2\ell}\le-\frac{L^2}{2\ell}-\frac{L\log L}{\ell}+\frac{(1+\log\ell)L}{\ell}.
\]
\emph{Lower bound.} Let $n:=\lceil(L+\log\frac{4}{1-q})/\ell\rceil$. Let $E_1:=\{e_k\le x_k\ \forall k<n\}$ with $x_k:=tq^{-k}/(2n)$, and $E_2:=\{\sum_{k\ge n}e_kq^k\le t/2\}$. These events are independent and $E_1\cap E_2\subseteq\{\zb\le t\}$. By Markov's inequality, $\P(E_2^c)\le 2q^n/((1-q)t)\le\frac12$. Since $1-e^{-x}\ge\frac12\min(x,1)$,
\[
\log\P(E_1)\ge-n\log2+\sum_{k<n}\min(0,k\ell-M),\qquad M:=L+\log(2n).
\]
The function $k\mapsto M-k\ell$ is decreasing, so the sum over $0\le k\le M/\ell$ of $M-k\ell$ is at most $M+M^2/(2\ell)$. Hence
\[
\log\P(\zb\le t)\ \ge\ -\log 2-n\log2-M-\frac{M^2}{2\ell}.
\]
Now $n=L/\ell+O(1)$ and $\log(2n)=\log L+O(1)$, so $M^2=L^2+2L\log L+O(L)$, and the lower bound matches the upper bound up to $O(L)$.

\ref{it:mellin} For $\Re s>-1$, integrate the density in \ref{it:series} against $t^s$ term by term, using
\[
\int_0^\infty t^se^{-t/q^k}\dd t=\Gamma(1+s)\,q^{k(1+s)} .
\]
The interchange is justified by absolute convergence, since the series $\sum_k q^{k(k-1)/2+k(1+\Re s)}/(q;q)_k$ converges. Then
\[
\E\zb^{\,s}=\frac{\Gamma(1+s)}{(q;q)_\infty}\sum_{k\ge0}\frac{(-1)^kq^{k(k-1)/2}(q^{1+s})^k}{(q;q)_k}=\frac{\Gamma(1+s)(q^{1+s};q)_\infty}{(q;q)_\infty}
\]
by \eqref{eq:euler2}. By \ref{it:fast} and \ref{it:slow}, $\E\zb^{\,x}<\infty$ for every real $x$. So $s\mapsto\E e^{s\log\zb}$ is the moment generating function of $\log\zb$, finite on all of $\R$, and therefore entire. On the right-hand side, the pole of $\Gamma(1+s)$ at $s=-n$ is cancelled by the zero of the factor $1-q^{s+n}$ in $(q^{1+s};q)_\infty$. The identity theorem extends the formula to $\C$. Near $s=-n$ we have $\Gamma(1+s)\sim(-1)^{n-1}/((n-1)!(s+n))$ and $1-q^{s+n}\sim\ell(s+n)$. The remaining factors give $\prod_{m=1}^{n-1}(1-q^{-m})\cdot(q;q)_\infty$, and combining these yields the stated negative moments.
\end{proof}

\begin{corollary}\label{cor:moments}
$\E\zb=\dfrac1{1-q}$, $\Var\zb=\dfrac1{1-q^2}$, and $\CV(\zb)=\sqrt{\dfrac{1-q}{1+q}}$. As $q\uparrow1$ (small ideas), $\CV\sim\sqrt{\ell/2}\to0$, consistent with \cref{prop:fluid}. As $q\downarrow0$ (one giant idea), $\CV\to1$.
\end{corollary}

\begin{remark}[The working paper's computation is a mean]\label{rem:casp-mean}
Chan et al.\ assume that each doubling of $A$ takes $2^{-\eps}\approx0.76$ times as long as the previous one, starting from $4.5$ months, and sum nine terms of the geometric series \cite[SM]{CASP26}. For a Good process with doubling-sized ideas ($h=1$) under convention~B, $q_0=1/d_0$ and the expected duration of the $k$-th doubling is exactly $d_0q^k$ with $q=2^{-\eps}$. So their seventeen months is, up to rounding, $\E T_9=\NumETnine$ months, and the full series $d_0/(1-2^{-\eps})=\NumEzetaC$ months is $\E\zeta$. \Cref{thm:law} supplies the rest of the distribution (\cref{tab:law}).
\end{remark}

Ideas need not all be the same size. For general sizes the explosion time is a \emph{perpetuity} \cite{Ver79}, and its moments remain explicit.

\begin{theorem}[Random idea sizes]\label{thm:random}
Let $\eps>0$, let $\mu$ be any i.i.d.\ strictly positive idea-size law, $Y:=(1+H)^{-\eps}$, $y_j:=\E Y^j\in(0,1)$, and $\zb:=\kappa A_0^\eps\zeta$. Then
\begin{enumerate}[label=(\alph*)]
\item $\zb\overset{d}{=}e+Y\zb'$ with $e\sim\Exp(1)$, $Y$ and $\zb'\overset d=\zb$ independent;
\item $\E\zb^{\,m}=m!\big/\prod_{j=1}^m(1-y_j)$ for $m\in\N$;
\item $\CV(\zb)^2=\dfrac{2(1-y_1)}{1-y_2}-1$.
\end{enumerate}
\end{theorem}

\begin{proof}
(a) $\zb=e_0+Y_1\sum_{n\ge1}e_n\prod_{2\le i\le n}Y_i$, and the last sum is a copy of $\zb$ independent of $(e_0,Y_1)$. (b) All moments are finite by \cref{thm:dichotomy}\ref{it:dich-power}. Expanding $(e+Y\zb')^m$ and isolating the term $j=m$ gives
\[
(1-y_m)\E\zb^{\,m}=\sum_{j=0}^{m-1}\frac{m!}{j!}\,y_j\,\E\zb^{\,j}\qquad(y_0:=1).
\]
The claim follows by induction from the telescoping identity $\sum_{j=0}^{m-1}y_j\prod_{i=1}^j(1-y_i)^{-1}=\prod_{i=1}^{m-1}(1-y_i)^{-1}$; to verify it, add the term $j=m$ to both sides and use $(1-y_m)+y_m=1$. (c) follows from (b) with $m=1,2$.
\end{proof}

\begin{corollary}[Equal ideas are the most predictable]\label{cor:lattice-min}
Among all idea-size laws with the same standardised mean explosion time $\E\zb=1/(1-y_1)$ (equivalently, the same mean at a common current rate $\kappa A_0^{\eps}$), the coefficient of variation of $\zb$ is minimised exactly when $Y$ is a.s.\ constant, i.e.\ in the lattice case.
\end{corollary}

\begin{proof}
At fixed $y_1$, $\CV^2$ is increasing in $y_2$, and $y_2\ge y_1^2$, with equality if and only if $Y$ is a.s.\ constant.
\end{proof}

The perpetuity is explicit for at least one natural family of sizes.

\begin{corollary}[Lomax ideas give a Gamma explosion time]\label{cor:gamma}
Let $\eps>0$, $b>0$, and suppose $\P(H>x)=(1+x)^{-b\eps}$ for $x\ge0$; equivalently, $\log(1+H)$ is exponential with rate $b\eps$, or $Y=(1+H)^{-\eps}\sim\BetaDist(b,1)$. Then
\[
\zb\sim\GammaDist(b+1,1),\qquad \E\zb=b+1,\qquad \CV(\zb)=\frac1{\sqrt{b+1}} .
\]
At the same standardised mean the lattice case has $\CV=1/\sqrt{2b+1}$, in line with \cref{cor:lattice-min}. The improvement size has finite mean if and only if $b\eps>1$; the corollary needs no such condition.
\end{corollary}

\begin{proof}
$\P(Y\le y)=\P\big(1+H\ge y^{-1/\eps}\big)=y^{b}$, so $Y\sim\BetaDist(b,1)$ and $y_j=b/(b+j)$. By \cref{thm:random}(b),
\[
\E\zb^{\,m}=m!\prod_{j=1}^m\frac{b+j}{j}=\frac{\Gamma(b+m+1)}{\Gamma(b+1)} ,
\]
which are the moments of $\GammaDist(b+1,1)$. The Gamma law is determined by its moments (its moment generating function is finite near $0$), so $\zb\sim\GammaDist(b+1,1)$. As a direct check, $F(s)=(1+s)^{-(b+1)}$ satisfies the transform form of \cref{thm:random}(a), $F(s)=(1+s)^{-1}\E F(sY)$, because $\frac{\dd}{\dd y}\big[y^b(1+sy)^{-b}\big]=b\,y^{b-1}(1+sy)^{-b-1}$ gives $\int_0^1 b\,y^{b-1}(1+sy)^{-(b+1)}\dd y=(1+s)^{-b}$. For the lattice comparison, $\E\zb=b+1$ forces $q=b/(b+1)$, and $(1-q)/(1+q)=1/(2b+1)$.
\end{proof}

\begin{remark}[A Ramanujan integral]\label{rem:ramanujan}
For $0<u<1$, Tonelli's theorem applied to $\zb^{-u}=\Gamma(u)^{-1}\int_0^\infty s^{u-1}e^{-s\zb}\dd s$, together with \cref{thm:law}\ref{it:laplace} and \ref{it:mellin} at $s=-u$, gives
\[
\int_0^\infty\frac{t^{u-1}\dd t}{(-t;q)_\infty}=\frac{\pi}{\sin\pi u}\,\frac{(q^{1-u};q)_\infty}{(q;q)_\infty},
\]
an evaluation in the circle of Ramanujan's $q$-beta integrals \cite{Ask80,GR04}.
\end{remark}

\section{Residual-time forecasting and parameter uncertainty}\label{sec:midnight}

\subsection{Self-similarity of the residual time}

\begin{theorem}[Residual-time self-similarity]\label{thm:selfsim}
In the lattice case with $\eps>0$, let $\tau$ be any $(\F_t)$-stopping time. On $\{\tau<\zeta\}$, conditionally on $\F_\tau$, the random variable $q_{N_\tau}(\zeta-\tau)$ has law $\Lambda_q$. That is, for every bounded measurable $\varphi$,
\[
\E\big[\varphi\big(q_{N_\tau}(\zeta-\tau)\big)\mathbf 1_{\{\tau<\zeta\}}\,\big|\,\F_\tau\big]=\mathbf 1_{\{\tau<\zeta\}}\,\E\varphi(\zb),
\]
where expressions multiplied by $\mathbf 1_{\{\tau<\zeta\}}$ are read as zero outside that event.
\end{theorem}

\begin{proof}
$(N_t)$ is the minimal pure birth chain with rates $(q_n)$. By the strong Markov property of minimal continuous-time Markov chains \cite{Nor97}, on $\{\tau<\zeta,\ N_\tau=n\}$ and conditionally on $\F_\tau$, the process $(N_{\tau+s}-n)_{s\ge0}$ is a minimal pure birth chain with rates $(q_{n+j})_{j\ge0}=(\rho^nq_j)_{j\ge0}$, and its explosion time is $\zeta-\tau$. Multiplying all rates by $\rho^n$ divides all holding times by $\rho^n$. So $\zeta-\tau\overset d=\rho^{-n}\zeta=\zb/q_n$ conditionally, as claimed.
\end{proof}

\paragraph{What the theorem does and does not say.} With model parameters fixed and known, the remaining time has a scale-invariant conditional distribution, so its \emph{relative} predictive uncertainty does not vanish as the process approaches explosion. Absolute uncertainty, measured in calendar time, does shrink, in proportion to the current expected waiting time. The theorem does not rule out learning unknown parameters, which is a different source of uncertainty from the irreducible randomness of future holding times, and it is not a statement about intelligence-explosion models outside this class. Let $m_t:=1/q_{N_t}$ be the current expected waiting time for the next idea.

\begin{corollary}\label{cor:forecast}
On $\{t<\zeta\}$, with parameters known:
\begin{enumerate}[label=(\alph*)]
\item $\E[\zeta-t\mid\F_t]=m_t/(1-q)$ and $\CV(\zeta-t\mid\F_t)=\sqrt{(1-q)/(1+q)}$, whatever $t$ and whatever has happened so far;
\item every $\F_t$-measurable forecast $\hat Z$ satisfies the conditional mean-squared-error bound
\[
\E[(\zeta-\hat Z)^2\mid\F_t]\ \ge\ \frac{1-q}{1+q}\,\big(\E[\zeta-t\mid\F_t]\big)^2 ,
\]
with equality at $\hat Z=\E[\zeta\mid\F_t]$;
\item if $[a,b]$ is an $\F_t$-measurable interval with $\P(\zeta\in[a,b]\mid\F_t)\ge1-\delta$, then
\[
b-a\ \ge\ (1-\delta)(1-q)\,\E[\zeta-t\mid\F_t].
\]
\end{enumerate}
\end{corollary}

\begin{proof}
(a) follows from \cref{thm:selfsim,cor:moments}. (b) holds because $\E[(\zeta-\hat Z)^2\mid\F_t]=\Var(\zeta\mid\F_t)+(\E[\zeta\mid\F_t]-\hat Z)^2$. For (c), conditionally on $\F_t$ the variable $q_{N_t}(\zeta-t)$ has law $\Lambda_q$, whose density is bounded by $1$ (it is the convolution of the $\Exp(1)$ density with another law). An interval carrying mass $1-\delta$ therefore has length at least $1-\delta$ in these units, and undoing the scaling with (a) gives (c). The bound in (c) is valid but generally not sharp; it does not identify the shortest predictive interval.
\end{proof}

\begin{remark}[Lead times in current units]\label{rem:leadtime}
Let $z_p$ be the $p$-quantile of $\Lambda_q$. By \cref{thm:selfsim}, $\P(\zeta-t\ge m_t z_p\mid\F_t)=1-p$ on $\{t<\zeta\}$. A response that needs lead time $\Delta$ and must be in place with probability at least $1-p$ therefore has to be launched while $m_t\ge\Delta/z_p$. For doubling-sized ideas at $\eps=\NumEps$, $m_t$ is the current expected doubling time, $\E[\zeta-t\mid\F_t]=\NumMeanUnits\,m_t$, and the $5\%$ and $95\%$ quantiles are $z_{0.05}=\NumZlo$ and $z_{0.95}=\NumZhi$. A twelve-month response that must be in place with probability $0.95$ has to start while $m_t\ge12/z_{0.05}=\NumLeadTwelve$ months. All of this is conditional on the model and its parameters being known.
\end{remark}

\begin{figure}[t]
\centering
\includegraphics[width=\textwidth]{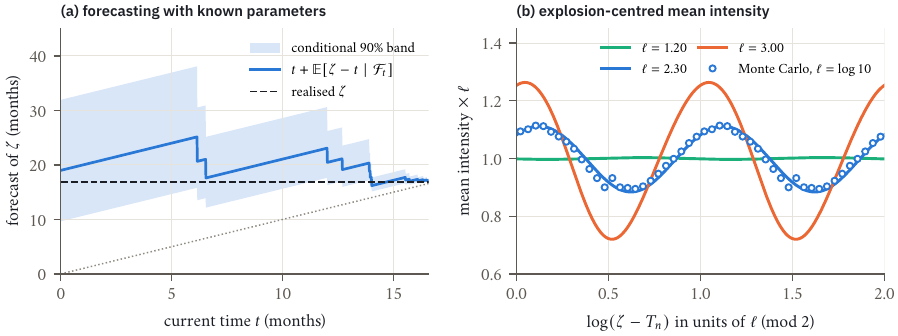}
\caption{(a) Forecasting along one simulated path ($h=1$, $\eps=\NumEps$, convention~B) with \emph{known} parameters. The forecast date $t+\E[\zeta-t\mid\F_t]$ jumps down at each idea and drifts up in between. The shaded region is the conditional $90\%$ predictive band given $\F_t$ under the model; it is not a confidence band for an unknown deterministic date. (b) Mean intensity of ideas in coordinates centred on the realised explosion time, $\log(\zeta-T_n)$, folded modulo $\ell$ (\cref{prop:logperiodic}). The curves are the exact Fourier series with $\Gamma$-coefficients; the circles are Monte Carlo for $\ell=\log10$. For $\ell=1.2$ the oscillation is about $\pm0.3\%$.}
\label{fig:forecast}
\end{figure}

\begin{remark}[Random sizes]
For non-lattice sizes, \cref{thm:selfsim} holds at the idea times $\tau=T_n$ directly from \cref{def:good}: $\kappa(A^{(n)})^{\eps}(\zeta-T_n)$ is independent of $\F_{T_n}$ with the law of $\zb$ in \cref{thm:random}. We do not use the extension to arbitrary stopping times.
\end{remark}

\subsection{Log-periodic structure in explosion-centred time}

Finite-time singularities with discrete scale invariance are known to produce log-periodic oscillations as the critical time is approached \cite{Sor98,Sor03}. In the lattice Good process the discrete scale invariance is exact: one more idea multiplies all future rates by $\rho$. The amplitude of the resulting oscillation is given exactly by the Gamma function.

\begin{proposition}[Explosion-centred mean intensity]\label{prop:logperiodic}
In the lattice case with $\eps>0$, the ideas arrive at $T_1,T_2,\dots$. Let $\nu:=\sum_{n\ge1}\delta_{\log(q_0(\zeta-T_n))}$ be the point process of their log-residual times, and let $\phi$ be the density of $\log\zb$. Then $\E\nu(\dd x)=m(x)\dd x$ with
\[
m(x)=\sum_{n\ge1}\phi(x+n\ell),\qquad |m(x)-p(x)|\le\frac{e^{x}}{1-q},
\]
where $p$ is the $\ell$-periodic function
\[
p(x)=\frac1\ell\Big[1+2\sum_{k\ge1}\Re\Big(\Gamma\big(1-\tfrac{2\pi ik}{\ell}\big)\,e^{2\pi ikx/\ell}\Big)\Big],\qquad
\Big|\Gamma\big(1-\tfrac{2\pi ik}{\ell}\big)\Big|^2=\frac{2\pi^2k/\ell}{\sinh(2\pi^2k/\ell)} .
\]
In particular, as $x\to-\infty$ the relative amplitude of the fundamental harmonic of $\ell\,m(x)$ tends to $2|\Gamma(1+2\pi i/\ell)|\sim(4\pi/\sqrt\ell)\,e^{-\pi^2/\ell}$ as $\ell\to0$.
\end{proposition}

\begin{proof}
From \eqref{eq:zeta}, $\zeta-T_n=q^n\zeta^{(n)}$ with $\zeta^{(n)}:=\sum_{j\ge0}e_{n+j}q^j/q_0\overset d=\zeta$. Hence $\log(q_0(\zeta-T_n))=-n\ell+\log\zb^{(n)}$, where each $\zb^{(n)}:=q_0\zeta^{(n)}$ has law $\Lambda_q$, and linearity of expectation gives $m$. The density of $\zb$ is bounded by $1$, so $\phi(y)=e^yf_{\zb}(e^y)\le e^y$, and
\[
|m(x)-p(x)|=\sum_{j\ge0}\phi(x-j\ell)\le e^x\sum_{j\ge0}q^j ,
\]
where $p(x):=\sum_{n\in\mathbb Z}\phi(x+n\ell)$. The periodisation $p$ of the integrable function $\phi$ has Fourier coefficients
\[
\frac1\ell\int_\R\phi(y)e^{-2\pi iky/\ell}\dd y=\frac1\ell\,\E\big[\zb^{-2\pi ik/\ell}\big]=\frac1\ell\,\Gamma\big(1-\tfrac{2\pi ik}\ell\big)\frac{(q^{1-2\pi ik/\ell};q)_\infty}{(q;q)_\infty}=\frac1\ell\,\Gamma\big(1-\tfrac{2\pi ik}\ell\big),
\]
by \cref{thm:law}\ref{it:mellin} and the identity $q^{-2\pi ik/\ell}=e^{2\pi ik}=1$. The classical formula $|\Gamma(1+iy)|^2=\pi y/\sinh(\pi y)$ shows that the coefficients are absolutely summable. So the Fourier series converges uniformly to $p$, which is continuous because the periodisation converges locally uniformly. Stirling's formula gives the asymptotics.
\end{proof}

This is an ensemble statement in coordinates centred on the realised explosion time. Since that time is not known to an online observer, the proposition is not itself a prospective detection procedure, and it says nothing about how visible the oscillation is along a single path; substituting an estimated explosion time for $\zeta$ would require a separate treatment of the estimation error. What it does identify is an exact discrete-scale-invariant structure, together with the size of its harmonics. That size is nonzero but exponentially small in $1/\ell$: the fundamental amplitude $2|\Gamma(1+2\pi i/\ell)|$ is about $\NumAmpHone$ for doubling-sized ideas at $\eps=\NumEps$, about $\NumAmpTwelve$ at $\ell=1.2$, and $\NumAmpLogTen\%$ only when each idea multiplies the rate of progress tenfold ($\ell=\log10$); see \cref{fig:forecast}(b).

\subsection{Uncertain growth exponent: an explicit $1/t$ tail}

The working paper stresses that the returns parameter is uncertain; one estimate has a $90\%$ credible interval for $r$ of $(0.727,2.094)$ \cite[note~6]{CASP26}. Uncertainty near the critical value has a sharp consequence.

\begin{proposition}[Exact Pareto tail under parameter uncertainty]\label{prop:uncertain}
In the lattice case with fixed $h$ and fixed current idea rate $q_0$, let $\eps$ be random, independent of $(e_n)$, and suppose its law has a density $\pi$ on a neighbourhood of $0$ that is continuous at $0$ with $\pi(0)>0$. Then $\E[\zeta\mid\eps>0]=\infty$ and
\[
\lim_{t\to\infty}t\,\P(\zeta>t\mid\eps>0)=\frac{\pi(0)}{q_0\log(1+h)\,\P(\eps>0)} .
\]
Since $q_0\log(1+h)=g$ is the local log drift (\cref{box:calibration}, B), the constant equals $\pi(0)/(g\,\P(\eps>0))$, which is also the limit of $t\,\P(T^\ast>t\mid\eps>0)$ for the fluid waiting time $T^\ast=1/(\eps g)$.
\end{proposition}

\begin{proof}
Write $a:=\log(1+h)$, $C:=(q_0a)^{-1}$, and $\zeta_\eps$ for the explosion time given $\eps>0$. \emph{Infinite mean.} $\E\zeta_\eps=1/(q_0(1-e^{-a\eps}))\ge1/(q_0a\eps)$, and $\int_0\pi(\eps)\eps^{-1}\dd\eps=\infty$.

\emph{Large $\eps$.} Chernoff's bound with $s=\frac12$ and the inequality $-\log(1-y/2)\le y$ on $[0,1]$ give $\P(\zb>x)\le e^{-x/2}\E e^{\zb/2}\le\exp\{-x/2+1/(1-q)\}$. Since $1/(1-q)\le1+1/(a\eps)$, every $\eps\ge4C/t$ satisfies $\P(\zeta_\eps>t)\le e^{1-q_0t/4}$, uniformly. The contribution of $\{\eps\ge4C/t\}$ to $t\,\P(\zeta>t,\eps>0)$ is therefore at most $t\,e^{1-q_0t/4}\to0$.

\emph{Small $\eps$.} For $t$ large, $4C/t$ lies in the neighbourhood where $\pi$ is a density, and the substitution $\eps=u/t$ gives
\[
t\,\P(\zeta>t,\ 0<\eps<4C/t)=\int_0^{4C}\pi(u/t)\,\P(\zeta_{u/t}>t)\dd u .
\]
For fixed $u>0$, \cref{cor:moments} gives $\E[\zeta_{u/t}/t]=1/(q_0t(1-e^{-au/t}))\to C/u$ and $\Var(\zeta_{u/t}/t)=1/(q_0^2t^2(1-e^{-2au/t}))\to0$, so $\P(\zeta_{u/t}>t)\to\mathbf 1_{\{u<C\}}$ for $u\ne C$. The integrand is bounded, since $\pi$ is bounded near $0$ by continuity, so dominated convergence gives the limit $\int_0^{4C}\pi(0)\mathbf 1_{\{u<C\}}\dd u=\pi(0)C$. Dividing by $\P(\eps>0)$ proves the claim. The fluid statement is $\P(0<\eps<1/(gt))\sim\pi(0)/(gt)$.
\end{proof}

For a log-normal prior on $r$ matching the interval above, with $\beta=1.01$ and convention~B, $\P(r>1)=\NumPexplode$. Conditional on $r>1$, the median fluid waiting time is $\NumPriorMedian$ months, $\P(T^\ast>5\text{ years})=\NumPgtFive$, $\P(T^\ast>10\text{ years})=\NumPgtTen$, and the tail constant of \cref{prop:uncertain} is $\NumParetoC$ months (\cref{fig:uncertainty}). These values validate the calculation under the selected prior; they do not validate the prior or the underlying interval. The intrinsic randomness of \cref{thm:law} has an exponential tail, so it is \emph{parameter} uncertainty that makes the expected waiting time infinite, even though explosion is certain given $\eps>0$.

\begin{figure}[t]
\centering
\includegraphics[width=0.54\textwidth]{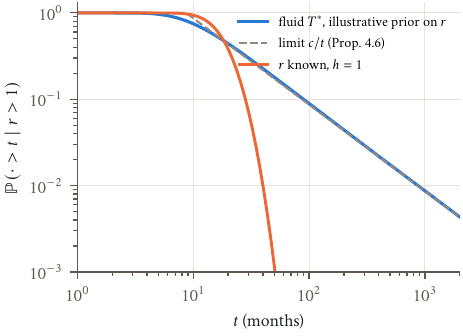}
\caption{Survival functions conditional on explosion. Blue: the fluid waiting time $T^\ast=1/(\eps g)$ mixed over an illustrative log-normal prior on $r$ that matches the interval $(0.727,2.094)$ of \cite[note~6]{CASP26}; this is not a fitted predictive distribution for actual AI development. Dashed: the limit $c/t$ of \cref{prop:uncertain}. Orange: intrinsic randomness only, with $r=\NumR$ known and doubling-sized ideas (\cref{thm:law}\ref{it:series}).}
\label{fig:uncertainty}
\end{figure}

\section{Fixed-sample detection and the cube-root law}\label{sec:detect}

The first priority in \cite{CASP26} is visibility, and one item on it is detecting the onset of an explosion. We study a specific, well-defined binary experiment, chosen to be as favourable to the observer as possible.

\paragraph{The experiment.} Both candidate processes start at a known origin. The monitor observes the first $n$ complete waiting times $(\tau_0,\dots,\tau_{n-1})$, so it waits until the $n$-th idea; the elapsed sampling time $T_n$ is itself random. The baseline rate $q_0$ and the alternative increment $\ell>0$ are known design parameters. After the $n$-th idea the monitor makes one decision between
\begin{align*}
&H_0:\ q_k\equiv q_0 &&(\text{steady exponential progress},\ r=1),\\
&H_1:\ q_k=q_0e^{k\ell} &&(\text{accelerating progress from the origin}).
\end{align*}
Let $P_0^n,P_1^n$ be the two laws: independent $\Exp(q_0)$, respectively independent $\Exp(q_0e^{k\ell})$. Define the minimax error and the minimum fixed-sample size for this experiment,
\[
\err_n(\ell):=\inf_\psi\max\big\{P_0^n(\psi=1),\,P_1^n(\psi=0)\big\},\qquad n^\ast(\alpha,\ell):=\min\{n:\err_n(\ell)\le\alpha\},
\]
where the infimum runs over all (possibly randomised) tests based on $(\tau_0,\dots,\tau_{n-1})$, and write $z_\alpha:=\Phi^{-1}(1-\alpha)$. This is not an anytime alarm and not the detection of an unknown change point: the alternative accelerates from the beginning of the record, and the error bounds below must not be reused for repeated testing after every idea, which needs its own false-alarm criterion.

\begin{theorem}[Cube-root law for the fixed-sample experiment]\label{thm:detect}
\leavevmode
\begin{enumerate}[label=(\roman*)]
\item\label{it:det-lower} For all $\ell>0$ and $\alpha\in(0,\frac14)$: \quad $n^\ast(\alpha,\ell)\ \ge\ \big(6\log\tfrac1{4\alpha}\big)^{1/3}\,\ell^{-2/3}$.
\item\label{it:det-upper} For all $\ell>0$ and $n\ge1$: \quad $\err_n(\ell)\le\exp\!\Big(-\dfrac{\ell^2(n-1)^3}{24(1+n\ell/2)}\Big)$. Hence $n^\ast(\alpha,\ell)\le n$ for every $n$ with $\ell^2(n-1)^3\ge24(1+n\ell/2)\log\frac1\alpha$.
\item\label{it:det-sharp} For every $\alpha\in(0,\frac12)$: \quad $\displaystyle\lim_{\ell\downarrow0}\ \ell^{2/3}\,n^\ast(\alpha,\ell)=(12z_\alpha^2)^{1/3}$.
\end{enumerate}
\end{theorem}

The exponent $2/3$ has a one-line explanation: the log-rate drifts by $\ell$ per idea, so after $n$ ideas the Kullback--Leibler divergence is $\asymp\ell^2\sum_{k<n}k^2\asymp\ell^2n^3$, and reliable discrimination needs this to be of order one.

\begin{lemma}\label{lem:logcosh}
For $x\ge0$, $\tanh x\ge x/(1+x)$, and consequently $\log\cosh x\ge x^2/(2(1+x))$.
\end{lemma}

\begin{proof}
Let $g(x):=(1+x)\tanh x-x$. Then $g(0)=0$ and $g'(x)=(1+x)\sech^2x-(1-\tanh x)=2\big((1+2x)e^{2x}-1\big)/(e^{2x}+1)^2\ge0$. Next let $f(x):=\log\cosh x-x^2/(2(1+x))$. Then $f(0)=0$ and $f'(x)=\tanh x-\frac{x(x+2)}{2(1+x)^2}\ge\frac{x}{1+x}-\frac{x(x+2)}{2(1+x)^2}=\frac{x^2}{2(1+x)^2}\ge0$.
\end{proof}

\begin{proof}[Proof of \cref{thm:detect}]
The log-likelihood ratio is
\[
\mathcal L_n:=\log\frac{\dd P_1^n}{\dd P_0^n}=\sum_{k<n}\big(k\ell-(e^{k\ell}-1)q_0\tau_k\big).
\]
Under $P_0$, $q_0\tau_k=E_k$, and under $P_1$, $q_0\tau_k=e^{-k\ell}E_k$, with $(E_k)$ i.i.d.\ $\Exp(1)$.

\ref{it:det-lower} By the Bretagnolle--Huber inequality \cite{BH79}, \cite[Lemma~2.6]{Tsy09}, every test satisfies $P_0^n(\psi=1)+P_1^n(\psi=0)\ge\frac12\exp(-\KL(P_1^n\|P_0^n))$. Moreover
\[
\KL(P_1^n\|P_0^n)=\sum_{k<n}\big(k\ell-1+e^{-k\ell}\big)\le\frac{\ell^2}{2}\sum_{k<n}k^2<\frac{\ell^2n^3}{6}.
\]
If both errors are at most $\alpha$, then $2\alpha\ge\frac12e^{-\ell^2n^3/6}$, which rearranges to the claim.

\ref{it:det-upper} Take $\psi:=\mathbf 1\{\mathcal L_n>0\}$. By Markov's inequality, $P_0^n(\mathcal L_n>0)\le\E_0e^{\mathcal L_n/2}$ and $P_1^n(\mathcal L_n\le0)\le\E_1e^{-\mathcal L_n/2}$. Both equal the Bhattacharyya coefficient $\prod_{k<n}\int\sqrt{p_{0,k}p_{1,k}}$. For exponentials with rates $q_0$ and $q_0e^{k\ell}$ each factor is $2e^{k\ell/2}/(1+e^{k\ell})=\sech(k\ell/2)$. By \cref{lem:logcosh},
\[
-\sum_{k<n}\log\sech\frac{k\ell}{2}\ \ge\ \sum_{k<n}\frac{(k\ell/2)^2}{2(1+k\ell/2)}\ \ge\ \frac{\ell^2}{8(1+n\ell/2)}\sum_{k<n}k^2\ \ge\ \frac{\ell^2(n-1)^3}{24(1+n\ell/2)},
\]
using $\sum_{k<n}k^2=(n-1)n(2n-1)/6\ge(n-1)^3/3$.

\ref{it:det-sharp} \emph{Step 1: reduction to the likelihood-ratio test.} For $n\ge2$, $\mathcal L_n$ has a continuous law under both hypotheses, as a nondegenerate linear combination of independent exponentials. By the Neyman--Pearson lemma \cite{LR05}, for every test $\psi$ there is a threshold $t$ with $P_0(\mathcal L_n>t)=P_0(\psi=1)$ and $P_1(\mathcal L_n\le t)\le P_1(\psi=0)$. Hence $\err_n=\inf_t\max\{P_0(\mathcal L_n>t),P_1(\mathcal L_n\le t)\}$. Also $\err_n$ is nonincreasing in $n$, since a test may ignore data.

\emph{Step 2: Gaussian limit.} Fix $c>0$, let $n=n(\ell):=\lfloor c\ell^{-2/3}\rfloor$ and $\delta:=n\ell$, so $\delta\to0$ and $\ell^2n^3\to c^3$ as $\ell\downarrow 0$. For $x\ge0$, Taylor's theorem gives $\frac{x^2}2\le e^x-1-x\le\frac{x^2}{2}e^x$ and $\frac{x^2}{2}e^{-x}\le x-1+e^{-x}\le\frac{x^2}2$, and also $x\le e^x-1\le xe^x$ and $xe^{-x}\le 1-e^{-x}\le x$. With $S_2:=\sum_{k<n}(k\ell)^2\to c^3/3$, these give the following. Under $P_0$, $\E_0\mathcal L_n=-\sum_{k<n}(e^{k\ell}-1-k\ell)\in[-e^\delta S_2/2,-S_2/2]$ and $\Var_0\mathcal L_n=\sum_{k<n}(e^{k\ell}-1)^2\in[S_2,e^{2\delta}S_2]$. Under $P_1$, $\E_1\mathcal L_n=\sum_{k<n}(k\ell-1+e^{-k\ell})\in[e^{-\delta}S_2/2,S_2/2]$ and $\Var_1\mathcal L_n=\sum_{k<n}(1-e^{-k\ell})^2\in[e^{-2\delta}S_2,S_2]$. The Lyapunov ratio is at most $\E|E_1-1|^3e^{3\delta}\sum_{k<n}(k\ell)^3/(e^{-3\delta}S_2^{3/2})\le\E|E_1-1|^3e^{6\delta}\,\delta\,S_2^{-1/2}\to0$. By Lyapunov's central limit theorem, $\mathcal L_n\Rightarrow\mathcal N(\mp c^3/6,\,c^3/3)$ under $P_0$ and $P_1$ respectively.

\emph{Step 3: conclusion.} The limits are continuous, so by P\'olya's theorem the distribution functions converge uniformly, and therefore
\[
\err_{n(\ell)}(\ell)\to\inf_t\max\Big\{\bar\Phi\Big(\frac{t+\mu}{\sigma}\Big),\Phi\Big(\frac{t-\mu}{\sigma}\Big)\Big\}=\Phi\Big(-\frac\mu\sigma\Big)=\Phi\Big(-\frac{c^{3/2}}{\sqrt{12}}\Big),
\]
with $\mu=c^3/6$ and $\sigma^2=c^3/3$; the infimum is attained at $t=0$ by symmetry. Let $c^\ast:=(12z_\alpha^2)^{1/3}$. If $c<c^\ast$, the limit exceeds $\alpha$, so $\err_{n(\ell)}>\alpha$ for small $\ell$, and by monotonicity $n^\ast>c\ell^{-2/3}-1$. If $c>c^\ast$, the limit is below $\alpha$, so $n^\ast\le c\ell^{-2/3}$ for small $\ell$. Since $c$ was arbitrary, $\ell^{2/3}n^\ast\to c^\ast$.
\end{proof}

How much time remains at the decision? The next statement answers this with a ratio of expectations and with predictive quantiles.

\begin{corollary}[Remaining duration at the decision]\label{cor:warning}
Under $H_1$, at the arrival time $T_n$ of the $n$-th idea, $\zeta-T_n$ is independent of $\F_{T_n}$ and distributed as $q^n\zeta$, where $q=e^{-\ell}$. Hence its quantiles are $q^n$ times the quantiles of $\zeta$, and
\[
W_n:=\frac{\E_1[\zeta-T_n]}{\E_1\zeta}=q^n .
\]
Let $W^\ast(\alpha,\ell):=W_{n^\ast(\alpha,\ell)}$. Then $W^\ast\le\exp\{-(6\log\frac1{4\alpha})^{1/3}\ell^{1/3}\}$ for all $\ell>0$ and $\alpha\in(0,\frac14)$, and
\[
\log W^\ast(\alpha,\ell)=-(12z_\alpha^2\,\ell)^{1/3}\,(1+o(1))\qquad(\ell\downarrow0).
\]
\end{corollary}

\begin{proof}
The first statement is \cref{thm:selfsim} at $\tau=T_n$, or \eqref{eq:zeta} directly. The bounds follow from $W^\ast=e^{-\ell n^\ast}$ and \cref{thm:detect}.
\end{proof}

$W^\ast$ is the remaining expected duration normalised by the initial expected explosion time. It is \emph{not} the expectation of the fraction remaining along a realised trajectory, $\E_1[(\zeta-T_n)/\zeta]$, whose denominator is random and contains the remaining time. For doubling-sized ideas and $n=9$, a Monte Carlo estimate of the latter is $\NumPathFrac\pm\NumPathFracSE$, whereas $W_9=\NumRatioMeans$ (\cref{app:numerics}).

\begin{figure}[t]
\centering
\includegraphics[width=\textwidth]{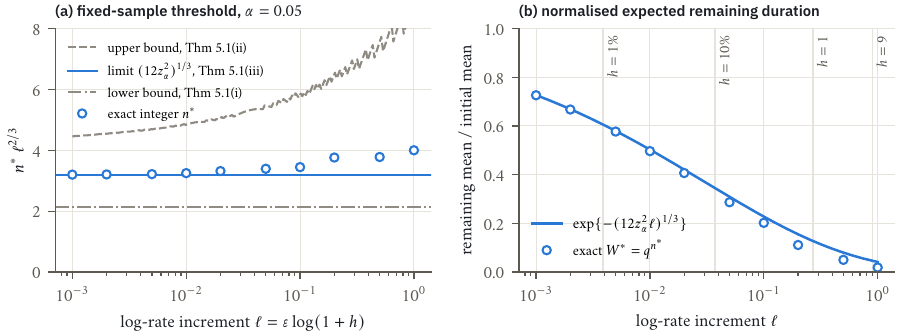}
\caption{The fixed-sample experiment of \cref{sec:detect}: known baseline rate, known increment $\ell$, one decision after the first $n$ waiting times. (a) The minimum fixed-sample size $n^\ast$ at $\alpha=0.05$, rescaled by $\ell^{2/3}$. Circles are exact integer thresholds from the hypoexponential law of the likelihood ratio (\cref{app:numerics}); the horizontal line is the limit $(12z_\alpha^2)^{1/3}\approx\NumLAN$; the grey curves are the non-asymptotic bounds of \cref{thm:detect}. (b) The normalised expected remaining duration $W^\ast=q^{n^\ast}$ (remaining mean divided by initial mean). Vertical lines mark idea sizes $h$ at $\eps=\NumEps$.}
\label{fig:detection}
\end{figure}

\paragraph{Reading the threshold.} Since $W^\ast\to1$ as $\ell\to0$, an explosion made of small ideas becomes distinguishable from steady growth, in this experiment, while most of the expected remaining duration is still ahead. The convergence is slow, at rate $\ell^{1/3}$, and for lumpy progress $W^\ast$ is small. \Cref{tab:detect} gives the calibrated values, including predictive quantiles of the remaining time. The data-processing inequality shows that any statistic of the same first-$n$ record, for instance a coarsened version of it, cannot do better than the threshold of \cref{thm:detect}. This says nothing about a real monitor that has auxiliary information absent from the model.

\begin{remark}[No uniform fixed-sample test over unseparated alternatives]\label{rem:unseparated}
If $\ell$ is unknown and only $\ell>0$ is specified, no fixed-$n$ test has a nontrivial uniform error bound. Indeed $\KL(P_\ell^n\|P_0^n)=\sum_{k<n}(k\ell-1+e^{-k\ell})\to0$ as $\ell\downarrow0$, so $P_\ell^n\to P_0^n$ in total variation by Pinsker's inequality. If $P_0^n(\psi=1)\le\alpha$, then $P_\ell^n(\psi=1)\to P_0^n(\psi=1)\le\alpha$, and the worst-case type-II error over $\ell>0$ is at least $1-\alpha$. Meaningful formulations use a separated alternative $\ell\ge\ell_{\min}>0$, a local asymptotic power criterion, or a sequential design with an explicit criterion.
\end{remark}

\begin{remark}[Unknown baseline: an invariant reduction]\label{rem:baseline}
Suppose $\ell$ is specified but $q_0$ is unknown. The lower bound \ref{it:det-lower} still holds, since knowing $q_0$ can only help. The scale-invariant statistic $P_k:=\tau_k/\sum_{j<n}\tau_j$ removes $q_0$ exactly: integrating out $S=\sum_j\tau_j$ (Jacobian $S^{n-1}$) shows that $P$ is uniform on the simplex under $H_0$ and that, under rates $q_0e^{k\ell}$,
\[
\frac{f_\ell(P)}{f_0(P)}=\frac{\exp\{\ell\sum_{k<n}k\}}{\big(\sum_{k<n}e^{k\ell}P_k\big)^n}.
\]
The score at $\ell=0$ is $U_n=\sum_{k<n}k-n\sum_{k<n}kP_k$, whose null variance, by the Dirichlet covariance formula, is $\frac{n}{n+1}\sum_{k<n}(k-\bar k)^2\sim n^3/12$, against $\sum_kk^2\sim n^3/3$ for known $q_0$. This suggests the constant $(48z_\alpha^2)^{1/3}$, about $1.59$ times as many ideas. We have not proved this: it requires a composite minimax formulation, triangular-array limits under both hypotheses and a matching lower bound (\cref{sec:open}).
\end{remark}

\section{Model extensions: delay, diversion and an oversight race}\label{sec:brakes}

The three results of this section concern different models with different assumptions. They are not properties of one unified controlled stochastic system.

\subsection{A deterministic delayed-feedback model}

Training runs take months, and an idea improves the research workforce only once it has been deployed. Suppose deployment takes time $D$, so that effective labour depends on $A(t-D)$ while the difficulty of finding new ideas depends on the current frontier $A(t)$:
\begin{equation}\label{eq:lag}
\dot A(t)=c\,A(t)^{1-\beta}\big(kA(t-D)\big)^{\lambda},\qquad t>0,\qquad A=\varphi\ \text{on }[-D,0].
\end{equation}

\begin{theorem}[Delayed feedback]\label{thm:lag}
Let $\beta,\lambda,c,k,D>0$ and let $\varphi:[-D,0]\to(0,\infty)$ be continuous.
\begin{enumerate}[label=(\roman*)]
\item In the deterministic delayed model \eqref{eq:lag}, there is a unique solution on $[-D,\infty)$; it is finite and increasing on $[0,\infty)$, so there is no finite-time blow-up, for any $r$.
\item If $r=\lambda/\beta>1$, then in the deterministic delayed model
\[
\lim_{t\to\infty}\frac{\log\log A(t)}{t}=\frac{\log r}{D}.
\]
\item With $D=0$ (and $r>1$) the solution blows up at $T^\ast=(\eps ck^\lambda A(0)^\eps)^{-1}$.
\item \emph{Stochastic delayed model.} Let ideas of size $h$ arrive at intensity $(c/h)k^\lambda A_{t-}^{-\beta}A_{t-D}^{\lambda}$, with $A_s=A_0$ for $s\le0$. Then the process does not explode: $N_t<\infty$ for all $t$, a.s.
\end{enumerate}
\end{theorem}

\begin{proof}
Put $B:=A^\beta$ and $C:=\beta ck^\lambda$. Then \eqref{eq:lag} becomes
\begin{equation}\label{eq:lagB}
\dot B(t)=C\,B(t-D)^{r}.
\end{equation}
The factor $A(t)^{1-\beta}$ has cancelled: diminishing returns at the frontier exactly offset the frontier's own growth.

(i) On $[jD,(j+1)D]$ the right-hand side of \eqref{eq:lagB} is a known continuous function, so $B$ is obtained by integration, step by step (the method of steps \cite{HVL93}). Conversely, any positive $C^1$ solution of \eqref{eq:lag} gives a solution of \eqref{eq:lagB}, which proves uniqueness. Since $\dot B>0$, $B$ is increasing on $[0,\infty)$, and $B(t)\ge B(0)+C(\min\varphi)^{\beta r}t\to\infty$.

(ii) \emph{Upper bound.} For $t\ge D$, $B(t+D)=B(t)+C\int_{t-D}^{t}B(v)^r\dd v\le B(t)+CDB(t)^r$. Once $B\ge1$, this gives $\log B(t+D)\le r\log B(t)+\log(1+CD)$. So $u:=\log B+\log(1+CD)/(r-1)$ satisfies $u(t+D)\le ru(t)$, and therefore $\limsup t^{-1}\log\log B(t)\le(\log r)/D$.

\emph{Lower bound.} Fix $\delta\in(0,1)$. For $t\ge D$,
\[
B(t+D)\ \ge\ C\int_{t+D-\delta D}^{t+D}B(u-D)^r\dd u\ \ge\ C\delta D\,B(t-\delta D)^r .
\]
So $v:=\log B+\log(C\delta D)/(r-1)$ satisfies $v(s+(1+\delta)D)\ge r\,v(s)$ for $s\ge D$. Since $B\to\infty$, $v$ is bounded below by some $v_{\min}>0$ on a window of length $(1+\delta)D$. Iterating gives $\liminf t^{-1}\log\log B(t)\ge(\log r)/((1+\delta)D)$, and we let $\delta\downarrow0$. Finally $\log\log A=\log\log B-\log\beta$.

(iii) This is \eqref{eq:ode} with $c$ replaced by $ck^\lambda$.

(iv) On $(jD,(j+1)D]$, $A_{t-}\ge A_{jD}$ and $A_{t-D}\le A_{jD}$. So the intensity is bounded by $Q_j:=(c/h)k^\lambda A_{jD}^{\lambda-\beta}$, which is $\F_{jD}$-measurable. By a thinning construction, conditionally on $\F_{jD}$ the number of ideas in that window is dominated by a Poisson variable with mean $Q_jD$, which is finite a.s. Induction on $j$ completes the proof.
\end{proof}

\Cref{thm:lag} makes precise, for this delayed-feedback model, the working paper's remark that some research steps are inherently sequential \cite[SM]{CASP26}. In the deterministic delayed model any fixed $D>0$ removes the singularity, and what remains is a tower-like growth law: the exponent of $A$ is multiplied by $r$ every $D$. With $D=3$ months and $r=\NumR$, $\log A$ doubles about every $\NumLagDouble$ months (\cref{fig:frictions}(a)). The proof uses the particular placement of the delayed labour term in \eqref{eq:lag}; it is not a statement about systems with an additional instantaneous feedback channel. For the stochastic delayed model we prove non-explosion only; whether $t^{-1}\log\log A_t\to(\log r)/D$ almost surely is open (\cref{sec:open}).

\subsection{Progressive diversion}

Suppose a share $\theta$ of the automated workforce is diverted from capabilities research to other uses, for example safety, oversight or beneficial applications.

\begin{proposition}[Progressive diversion]\label{prop:diversion}
Let $\dot A=cA^{1-\beta}\big((1-\theta(A))kA\big)^{\lambda}$, where the share retained for capabilities research is $1-\theta(a)=(a/A_0)^{-\gamma}$ for $a\ge A_0$, with $\gamma\ge0$ and $r>1$. Then $A$ explodes in finite time if and only if $\gamma<1-1/r$. At $\gamma=1-1/r$ growth is exponential, and for $\gamma>1-1/r$ it is polynomial. The same threshold holds for the Good process with rate $\propto a^{\eps-\gamma\lambda}$, by \cref{thm:dichotomy}.
\end{proposition}

\begin{proof}
The equation reads $\dot A=ck^\lambda A_0^{\gamma\lambda}A^{1+\eps_\gamma}$ with $\eps_\gamma:=\lambda(1-\gamma)-\beta$. It blows up if and only if $\eps_\gamma>0$, i.e.\ $\gamma<1-\beta/\lambda$. If $\eps_\gamma=0$ the growth is exponential, and if $\eps_\gamma<0$ then $A^{|\eps_\gamma|}$ grows linearly.
\end{proof}

At the working paper's central values $1-1/r=\NumGammaStar$: a diverted share of $\NumShareTen\%$ at ten times today's capability and $\NumShareThousand\%$ at a thousand times turns the singularity into exponential growth. A constant share $\theta<1$ only delays the blow-up, since $T^\ast$ scales like $(1-\theta)^{-\lambda}$. \Cref{prop:diversion} is a statement about the capability equation alone. Non-explosion is not a statement about whether oversight keeps pace, which depends on a separate equation (\cref{ex:diversion}).

\subsection{Fixed-allocation race between capabilities and oversight}

Let $S$ measure oversight capability (alignment, control, monitoring), produced by the diverted labour under its own idea production function. With a constant share $\theta\in(0,1)$:
\begin{equation}\label{eq:race}
\dot A=cA^{1-\beta}\big((1-\theta)kA\big)^\lambda,\qquad \dot S=c_SS^{1-\beta_S}\big(\theta kA\big)^{\lambda_S},\qquad A(0)=A_0,\ S(0)=S_0>0,
\end{equation}
with $\eps=\lambda-\beta>0$, $\beta_S,\lambda_S>0$, and $\eps_S:=\lambda_S-\beta_S$, the exponent with which oversight would explode if it were fed by its own output. Let $T^\ast$ be the blow-up time of $A$.

\begin{theorem}[Fixed-allocation asymptotic race regimes]\label{thm:race}
As $t\uparrow T^\ast$:
\begin{enumerate}[label=(\roman*)]
\item if $\lambda_S<\eps$, then $S(t)\to S(T^\ast)<\infty$;
\item if $\lambda_S\ge\eps$, then $S(t)\to\infty$, and
\[
\frac{S(t)}{A(t)}\longrightarrow
\begin{cases}
0,& \eps_S<\eps,\\[2pt]
\mathcal R(\theta):=\Big(\dfrac{c_S(\theta k)^{\lambda_S}}{c((1-\theta)k)^{\lambda}}\Big)^{1/\beta_S},&\eps_S=\eps,\\[8pt]
\infty,&\eps_S>\eps.
\end{cases}
\]
\end{enumerate}
Away from the critical equality $\eps_S=\eps$, changing a fixed allocation share $\theta\in(0,1)$ changes constants and dates but not the asymptotic regime of the ratio $S/A$. At the critical equality the limit $\mathcal R(\theta)$ depends on $\theta$.
\end{theorem}

\begin{proof}
Write $c':=c((1-\theta)k)^\lambda$. Then $A(t)=(\eps c'(T^\ast-t))^{-1/\eps}$, and $\frac{\dd}{\dd t}S^{\beta_S}=\beta_Sc_S(\theta kA)^{\lambda_S}$. Hence
\[
S(t)^{\beta_S}=S_0^{\beta_S}+\beta_Sc_S(\theta k)^{\lambda_S}(\eps c')^{-p}\int_0^t(T^\ast-u)^{-p}\dd u,\qquad p:=\lambda_S/\eps .
\]
If $p<1$ the integral converges as $t\uparrow T^\ast$, which is (i). If $p=1$ it diverges logarithmically, so $S\to\infty$ while $S/A\to0$; this case has $\eps_S=\eps-\beta_S<\eps$. If $p>1$, then $S^{\beta_S}=K(T^\ast-t)^{1-p}(1+o(1))$ with $K=\beta_Sc_S(\theta k)^{\lambda_S}(\eps c')^{-p}/(p-1)$, and so
\[
\frac{S(t)}{A(t)}=K^{1/\beta_S}(\eps c')^{1/\eps}\,(T^\ast-t)^{\frac{\eps-\eps_S}{\eps\beta_S}}\,(1+o(1)).
\]
The exponent is positive, zero or negative according as $\eps_S<\eps$, $\eps_S=\eps$ or $\eps_S>\eps$. In the case of equality, $p-1=\beta_S/\eps$, and a short computation gives $K^{1/\beta_S}(\eps c')^{1/\eps}=(c_S(\theta k)^{\lambda_S}/c')^{1/\beta_S}=\mathcal R(\theta)$.
\end{proof}

\begin{example}[Removing the singularity does not make oversight keep pace]\label{ex:diversion}
Take $\lambda=2$, $\beta=1$, $c=k=c_S=1$, $A_0=1$ and $1-\theta(A)=A^{-1/2}$, which is exactly the threshold $\gamma=1-1/r=\frac12$ of \cref{prop:diversion}. Then $\dot A=A$, so $A(t)=e^t$ and there is no blow-up. Let oversight follow the time-varying-allocation version of \eqref{eq:race} with $\lambda_S=1$ and $\beta_S=2$, namely $\dot S=S^{-1}\theta(A)A$. Then $\frac{\dd}{\dd t}S^2=2(e^t-e^{t/2})$, so
\[
S(t)^2=S_0^2+2(e^t-1)-4(e^{t/2}-1),\qquad \frac{S(t)}{A(t)}\sim\sqrt2\,e^{-t/2}\to0 ,
\]
although the diverted share tends to one.
\end{example}

\begin{figure}[t]
\centering
\includegraphics[width=\textwidth]{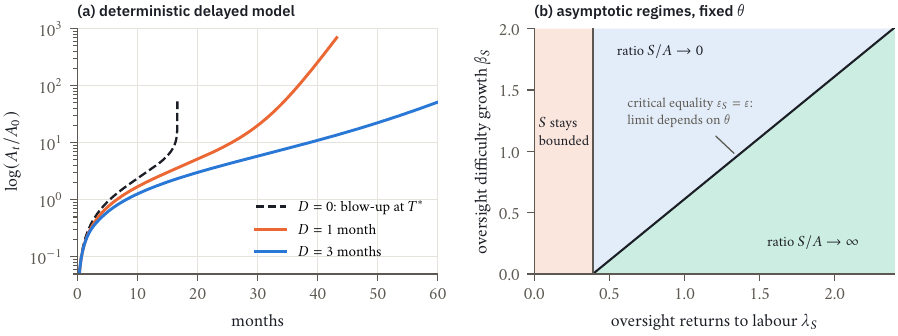}
\caption{(a) Solutions of the deterministic delayed model \eqref{eq:lag} with $D=1$ and $3$ months ($\lambda=1.40$, $\beta=1.01$, $c k^\lambda=g$, $A=1$ on $[-D,0]$): $\log A$ grows exponentially instead of blowing up. The stochastic delayed model is only proved not to explode. (b) Asymptotic regimes of the ratio $S/A$ in \cref{thm:race} at $\eps=\NumEps$ for a fixed allocation share $\theta$. On the boundary line the limit exists and depends on $\theta$.}
\label{fig:frictions}
\end{figure}

\paragraph{Reading the extensions.} Three qualifications belong next to any policy reading of this section. (a) The asymptotic regime of $S/A$ is insensitive to a fixed allocation share \emph{away from} the critical equality of exponents; at that equality the share sets the limiting ratio. (b) Non-explosion of $A$, as in \cref{prop:diversion}, is not the same as oversight keeping pace (\cref{ex:diversion}); the coupled problem with time-varying allocation is not solved here. (c) A policy can materially change finite-horizon performance, the time available for intervention and threshold crossings even when it does not change an asymptotic exponent. Treating $S/A$ as a measure of the adequacy of oversight is a modelling assumption, not a theorem about alignment. In particular, nothing here shows that fixed safety spending is useless or that diversion alone establishes safety. What the model suggests is that the exponents of oversight research, namely its returns to labour $\lambda_S$ and its difficulty growth $\beta_S$, deserve measurement alongside its budget.

\section{Numerical illustrations}\label{sec:calibration}

We use $\lambda=1.40$, $\beta=1.01$ ($\eps=\NumEps$, $r=\lambda/\beta$ unrounded) and convention~B of \cref{box:calibration} with $d_0=4.5$ months, so that $g=\log2/d_0$ and $q_0=g/\log(1+h)$ for each idea size $h$. The fluid blow-up time is $T^\ast=1/(\eps g)=\NumTfluid$ months. Under convention~B the columns $\E\zeta$, the quantiles, $\P(\zeta\le12)$ and all durations in months depend on the convention; $\CV$, $n^\ast$ and $W^\ast$ do not. These are illustrations of the theorems, not forecasts of AI development.

\begin{table}[t]
\caption{Exact laws of the explosion time (\cref{thm:law}) at $\eps=\NumEps$ under convention~B. Times in months. The value $\P(\zeta\le12)\approx1.2\times10^{-13}$ for $h=0.01$ is a high-precision evaluation, stable under precision doubling; it is far below Monte Carlo resolution.}
\label{tab:law}
\centering\small
\renewcommand{\arraystretch}{1.15}
\begin{tabular}{@{}lcccccc@{}}
\toprule
idea size $h$ & $\ell$ & $\E\zeta$ & $\CV$ & $5\%$ quantile & $95\%$ quantile & $\P(\zeta\le12)$\\
\midrule
$0.01$ & $\NumEllA$ & $\NumEzA$ & $\NumCVA$ & $\NumQloA$ & $\NumQhiA$ & $\NumPtwelveA$\\
$0.1$ & $\NumEllB$ & $\NumEzB$ & $\NumCVB$ & $\NumQloB$ & $\NumQhiB$ & $\NumPtwelveB$\\
$1$ & $\NumEllC$ & $\NumEzC$ & $\NumCVC$ & $\NumQloC$ & $\NumQhiC$ & $\NumPtwelveC$\\
$9$ & $\NumEllD$ & $\NumEzD$ & $\NumCVD$ & $\NumQloD$ & $\NumQhiD$ & $\NumPtwelveD$\\
\bottomrule
\end{tabular}
\end{table}

\begin{table}[t]
\caption{The fixed-sample experiment of \cref{sec:detect} at $\alpha=5\%$ (known $q_0$ and $\ell$). $n^\ast$ is the exact minimum fixed-sample size; the boundary pair $\err_{n^\ast-1}>\alpha\ge\err_{n^\ast}$ certifies it. $W^\ast=q^{n^\ast}$ is the remaining mean divided by the initial mean. The last three columns describe $\zeta-T_{n^\ast}$, which is distributed as $q^{n^\ast}\zeta$ (\cref{cor:warning}). Times in months.}
\label{tab:detect}
\centering\small
\renewcommand{\arraystretch}{1.15}
\begin{tabular}{@{}lcccccccc@{}}
\toprule
 & & & & & & \multicolumn{3}{c}{remaining time $\zeta-T_{n^\ast}$}\\ \cmidrule(l){7-9}
$h$ & $n^\ast$ & $\err_{n^\ast-1}$ & $\err_{n^\ast}$ & $W^\ast$ & $\E T_{n^\ast}$ & mean & $5\%$ quantile & $95\%$ quantile\\
\midrule
$0.01$ & $\NumNstarA$ & $\NumErrPrevA$ & $\NumErrNA$ & $\NumWA$ & $\NumETnA$ & $\NumRemA$ & $\NumRemLoA$ & $\NumRemHiA$\\
$0.1$ & $\NumNstarB$ & $\NumErrPrevB$ & $\NumErrNB$ & $\NumWB$ & $\NumETnB$ & $\NumRemB$ & $\NumRemLoB$ & $\NumRemHiB$\\
$1$ & $\NumNstarC$ & $\NumErrPrevC$ & $\NumErrNC$ & $\NumWC$ & $\NumETnC$ & $\NumRemC$ & $\NumRemLoC$ & $\NumRemHiC$\\
$9$ & $\NumNstarD$ & $\NumErrPrevD$ & $\NumErrND$ & $\NumWD$ & $\NumETnD$ & $\NumRemD$ & $\NumRemLoD$ & $\NumRemHiD$\\
\bottomrule
\end{tabular}
\end{table}

\paragraph{Reading the tables.} The deterministic $\NumTfluid$ months is the small-idea limit of a distribution, which is tight if progress consists of many small ideas and wide and right-skewed if it consists of a few large ones (\cref{tab:law}). At equal local log drift, lumpy progress also puts more mass on early arrival, $14$--$20\%$ within a year; unlike the $\CV$, this comparison depends on the calibration convention. With parameters known, relative uncertainty about the remaining time does not shrink (\cref{thm:selfsim}), so any plan keyed to this model is naturally indexed to the current expected doubling time rather than to a calendar date. In the fixed-sample experiment (\cref{tab:detect}) the attainable lead time is governed by the granularity $\ell$. For 1\% ideas the decision comes after about $\NumETnA$ months in expectation, with a remaining duration whose $5$--$95\%$ range is $[\NumRemLoA,\NumRemHiA]$ months. For doubling-sized ideas the decision comes after the working paper's tenfold acceleration ($\E T_9=\NumETnine$ months), with $[\NumRemLoC,\NumRemHiC]$ months remaining. These are predictive quantiles within the model, not guarantees about the world.

\section{Limitations and open problems}\label{sec:open}

The Good process is deliberately minimal (\cref{rem:scope}). Ideas are conditionally independent and strictly positive; the exponent $\eps$ is constant; compute and data, which the working paper lists as separate frictions, are absent; and the idea size $h$ is a free parameter without a direct empirical counterpart. The numerical illustrations depend on an explicit calibration convention and are not forecasts. Within these limits, the following problems seem natural.

\begin{enumerate}[leftmargin=*]
\item \emph{Composite and sequential detection.} With $\ell$ specified and $q_0$ unknown, prove the cube-root law through the invariant reduction of \cref{rem:baseline}, presumably with the constant $(48z_\alpha^2)^{1/3}$. With $\ell$ unknown, use a separated alternative $\ell\ge\ell_{\min}$ or a local power criterion (\cref{rem:unseparated}). Separately, determine optimal anytime alarms with explicit false-alarm control and treat an unknown onset (change point); neither is addressed by \cref{thm:detect}.
\item \emph{Multitype Good processes.} Software and hardware improve each other \cite{DHHK26}. For multitype pure-jump processes with rates $\kappa_i\prod_jA_j^{\eps_{ij}}$, find the explosion criterion and determine whether the explosion time is governed by a matrix analogue of $(z;q)_\infty$.
\item \emph{Stochastic delays.} For the stochastic delayed model of \cref{thm:lag}(iv), decide whether $t^{-1}\log\log A_t\to(\log r)/D$ almost surely, and describe the fluctuations.
\item \emph{Coupled allocation under uncertainty.} Choose a time-varying share $\theta_t$ when $\eps$ is unknown, with the coupled oversight equation of \cref{ex:diversion} in the loop. This is a control problem with learning and an absorbing catastrophe.
\item \emph{Exact laws beyond the lattice.} \Cref{cor:gamma} gives one explicit non-lattice law. Characterise broader classes of size laws for which the law of $\zb$ in \cref{thm:random} is explicit \cite{CPY97,BY05}.
\item \emph{Saturation.} With physical limits, $\eps$ eventually becomes negative \cite[note~5]{CASP26}. For logistic-type rate functions, describe the law of the time to reach a given fraction of the ceiling, and find what replaces the $q$-series.
\end{enumerate}

\paragraph{Acknowledgements.} We are grateful to Chan et al.\ \cite{CASP26} for a clear and carefully argued working paper; its explicit statement of the feedback model is what made the present analysis possible.

\paragraph{Use of AI tools.} AI-based assistants supported parts of this work, including the drafting and editing of the text, the development and cross-checking of the numerical code, and an additional technical reading of the manuscript. The authors reviewed this material and take full responsibility for the content of the paper.

\appendix
\section{Numerical specification and reproducibility}\label{app:numerics}

Every number in this paper is produced by the scripts in the \texttt{code/} folder of the accompanying repository~\cite{LGYS26},
\begin{center}
\url{https://github.com/xiaoyulics/singularity-distribution},
\end{center}
which also contains all generated outputs and figures; no number is typed by hand. \texttt{make\_numbers.py} computes all quantities and writes \texttt{generated/\allowbreak numbers.tex}, which the manuscript reads as \LaTeX\ macros. \texttt{run\_checks.py} runs the checks listed below and writes \texttt{generated/\allowbreak checks\_\allowbreak report.txt}; \texttt{make\_\allowbreak figures.py} regenerates \cref{fig:paths,fig:forecast,fig:uncertainty,fig:detection,fig:frictions}, and \texttt{make\_hero.py} with \texttt{make\_hero\_\allowbreak labels.py} regenerates \cref{fig:hero}. The script \texttt{reproduce.sh} (or \texttt{reproduce.ps1}) runs all of them in order. Environment: Python 3.13.14 with numpy 2.4.6, scipy 1.17.1, mpmath 1.3.0 and matplotlib 3.11.0 (PGF backend, so that figure text is set by \LaTeX\ in the paper's fonts); MiKTeX pdf\TeX\ 1.40.29. Random seeds are fixed in each script.

\paragraph{Exact law.} The alternating series of \cref{thm:law}\ref{it:series} is evaluated in \texttt{mpmath}. The working precision is set per evaluation from a floating-point estimate of the largest term relative to $(q;q)_\infty$ (the number of digits lost to cancellation) plus $40$--$60$ guard digits. For example, this gives $172$ digits for the $5\%$ quantile at $h=0.01$ ($q=0.99613$) and $90$ digits for $h\ge0.1$. Truncation is certified: the term magnitudes decrease monotonically once $q^{k+1}<\frac12$, i.e.\ for $k>\log2/\ell$, and from that point the remainder of the alternating series is bounded by the first omitted term, which is required to be below $10^{-(\text{digits}-5)}$ relative to the partial sum. Quantiles are found by bisection on the distribution function to relative tolerance $10^{-11}$. Lower-tail probabilities are computed as $1-\P(\zb>t)$ with additional guard digits, never subtracted at inadequate precision.

\paragraph{Checks run} (results in \texttt{generated/\allowbreak checks\_report.txt}).
\begin{itemize}[leftmargin=*]
\item \emph{Regression targets.} All entries of \cref{tab:law,tab:detect} agree with an independent recomputation to relative accuracy $2\times10^{-6}$, and the four detection boundary pairs agree to $10^{-9}$.
\item \emph{Precision doubling.} Each quantile and each $\P(\zeta\le12)$ is recomputed at twice the working precision; the largest change is below $10^{-40}$ in absolute terms, and the relative change of $\P(\zeta\le12)\approx1.16\times10^{-13}$ at $h=0.01$ is below $10^{-80}$. In the near-one regime $q=0.99613$, the normalisation $\P(\zb>0)=1$ holds to $10^{-39}$ at $312$ digits and to $10^{-351}$ at $624$ digits.
\item \emph{Structure.} On $400$-point grids the distribution function is monotone and the density is nonnegative within $10^{-14}$; for $h\in\{0.1,1,9\}$, moments $0$--$3$ obtained by quadrature of the density agree with $n!/(q;q)_n$ to relative accuracy $10^{-15}$.
\item \emph{Monte Carlo cross-checks with standard errors.} For $h\in\{0.1,1,9\}$ ($10^5$--$2\times10^5$ samples), the mean, the distribution function at the two quantiles and $\P(\zeta\le12)$ agree with the exact values within $4.5$ standard errors (largest $|z|$ about $2.3$). At $h=0.01$, $\P(\zeta\le12)$ is not Monte Carlo-validated: it is far below simulation resolution. The two error probabilities at the exact balancing threshold of the likelihood-ratio test are simulated ($4\times10^5$ replicates per hypothesis) and agree with the exact minimax error within $4.5$ standard errors. The pathwise fraction quoted after \cref{cor:warning} uses $2\times10^5$ paths truncated after $200$ ideas (the omitted tail has relative size below $10^{-23}$).
\item \emph{Detection thresholds.} For $k=1,\dots,n-1$ let $w_{0k}=e^{k\ell}-1$, $w_{1k}=1-e^{-k\ell}$ and $Z_j=\sum_kw_{jk}E_k$; the zero-weight term $k=0$ is omitted. Under $H_j$, $\mathcal L_n=\ell\sum_kk-Z_j$. With distinct weights, $\P(Z_j>x)=\sum_kc_{jk}e^{-x/w_{jk}}$ with $c_{jk}=\prod_{i\ne k}w_{jk}/(w_{jk}-w_{ji})$, evaluated at $300$ digits after checking $|\sum_kc_{jk}-1|<10^{-30}$. The balancing threshold solves $\P(Z_0>x)+\P(Z_1>x)=1$, and the common error is $\P(Z_1>x)$. We report the boundary pair $(\err_{n^\ast-1},\err_{n^\ast})$; at $h=0.01$, $\err_{n^\ast}$ lies only $8\times10^{-5}$ below $\alpha$, so a Monte Carlo search could not certify this threshold. The non-asymptotic bounds of \cref{thm:detect} bracket every computed $n^\ast$.
\item \emph{\Cref{prop:uncertain,cor:gamma}.} For $\eps\sim\mathrm{Uniform}(0,1)$, $h=1$, $q_0=1$, the quantity $t\,\P(\zeta>t\mid\eps>0)$, computed by Gauss--Legendre quadrature over $\eps$ with the exact survival function, equals $1.489$, $1.465$, $1.454$ and $1.448$ at $t=25,50,100,200$, approaching $1/\log2=1.4427$. For Beta$(b,1)$ marks, $b\in\{0.5,2,5\}$, simulated perpetuities pass a Kolmogorov--Smirnov test against Gamma$(b+1,1)$ ($10^5$ samples each), and the transform identity in the proof of \cref{cor:gamma} holds to $10^{-10}$.
\item \emph{Small-time asymptotic.} \texttt{check\_small\_time.py} evaluates $\log\P(\zb\le t)$ at $q=\frac12$ from the exact series at $600$--$18{,}000$ digits for $t=10^{-8},10^{-16},10^{-32},10^{-64}$. Relative to the three-term expansion of \cite[Theorem~1.1]{BR08} quoted after \cref{thm:law}, the residual divided by $L$ is $-0.81$, $-0.52$, $-0.33$ and $-0.20$, consistent with an $o(L)$ remainder; relative to the two terms of \cref{thm:law}\ref{it:fast} it stays bounded ($-0.40$ to $0.21$), as stated.
\item \emph{Gamma coefficient.} At $\ell=\log10$, quadrature of $\E\zb^{-2\pi i/\ell}$ against the exact density agrees with $\Gamma(1-2\pi i/\ell)$ to ten digits.
\item \emph{\Cref{fig:hero}.} \texttt{make\_hero.py} simulates $1\,400$ histories of a Good process with Lomax idea sizes, $\log(1+H)\sim\Exp(b\eps)$ with $b=6$ and $\eps=\NumEps$, at local log drift $\log2/4.5$ per month (so $q_0=gb\eps$ and $\zb\sim\GammaDist(7)$ by \cref{cor:gamma}). Histories are rasterised as additive light with bilinear splatting, bloom and a filmic tone map, and coloured by the rank of their explosion time. The histories at the $10\%$, $50\%$ and $90\%$ quantiles of $\zeta$ are traced brightly in both panels. Left panel: capability in doublings against calendar time ($0$--$36$ months, $0$--$13$ doublings); the dashed curve is the fluid ODE with $cA_0^{\eps}=g$; the horizon glow and silhouette are the exact density of $\zeta=\zb/q_0$; the ticks are at $\E T_n=(1-y_1^n)/(q_0(1-y_1))$ with $y_1=b/(b+1)$, which converge to $\E\zeta$. Right panel: radius $\zeta-t$ (up to $27$ months) and angle $2\pi\log_2(A_t/A_0)/5$, so that between ideas a history moves radially inwards and at each idea it turns through an arc. By the self-similarity of \cref{sec:midnight} (at idea times, as in the remark on random sizes), the part of a spiral after any idea is, in law, a rotated and rescaled copy of a whole spiral. The stars, the halo and the reflection in the water are decorative. \texttt{make\_hero\_labels.py} sets the labels in the paper's fonts.
\item \emph{Delay equation.} \eqref{eq:lagB} is integrated on the logarithmic scale by the explicit Euler method; the figure uses step $10^{-3}$ months. Halving the step from $2\times10^{-3}$ to $10^{-3}$ and to $5\times10^{-4}$ changes $\log A(t)$ at $t=10,20,30,40$ by at most $2.8\times10^{-4}$ and $1.4\times10^{-4}$ (relative) for $D=1$, and by $2.2\times10^{-5}$ and $1.1\times10^{-5}$ for $D=3$, consistent with first-order convergence.
\end{itemize}

\end{document}